%% file: PraC_tmlr_final.tex
\documentclass[10pt]{article} % For LaTeX2e
\usepackage[accepted]{tmlr}
\input{math_commands.tex}

\usepackage{hyperref}
\usepackage{url}
\usepackage{amsfonts} 
\usepackage{tcolorbox}
\usepackage{nicefrac}       % compact symbols for 1/2, etc.
\usepackage{microtype}      % microtypography
\usepackage{xcolor}         % colors
\usepackage{amsmath}
\usepackage{amssymb}
\usepackage{mathtools}
\usepackage{tabularray}
\usepackage{amsthm}
\usepackage{threeparttable}
\usepackage{multirow}
\usepackage{rotating}
\usepackage{booktabs}
\usepackage{tabularx}
\usepackage{array}
\usepackage{enumitem}
\usepackage{placeins}
\usepackage{graphicx}
\usepackage{siunitx}

\newcolumntype{Y}{>{\centering\arraybackslash}X}

\theoremstyle{plain}
\newtheorem{theorem}{Theorem}[section]

\newtheorem{lemma}[theorem]{Lemma}

\theoremstyle{definition}
\newtheorem{definition}[theorem]{Definition}

\theoremstyle{remark}

\newcommand{\ie}{\textit{i.e.}}
\newcommand{\eg}{\textit{e.g.}}

\title{Pragmatic Curiosity: A Unified Framework for Hybrid \\Learning and Optimization via Active Inference}

\author{\name Yingke Li \email yingkeli@mit.edu \\
      \addr Department of Aeronautics and Astronautics\\
      Massachusetts Institute of Technology
      \AND
      \name Anjali Parashar \email anjalip@mit.edu \\
      \addr Department of Aeronautics and Astronautics\\
      Massachusetts Institute of Technology
      \AND
      \name Enlu Zhou \email enlu.zhou@isye.gatech.edu\\
      \addr School of Industrial and Systems Engineering \\
      Georgia Institute of Technology
      \AND
      \name Chuchu Fan \email chuchu@mit.edu \\
      \addr Department of Aeronautics and Astronautics\\
      Massachusetts Institute of Technology
}

\def\month{11}  % Insert correct month for camera-ready version
\def\year{2026} % Insert correct year for camera-ready version
\def\openreview{\url{https://openreview.net/forum?id=DHkYR1qKpS}} % Insert correct link to OpenReview for camera-ready version

\begin{document}

\maketitle

\begin{abstract}
Many engineering and scientific workflows rely on expensive black-box evaluations, requiring sequential decisions that must both improve task performance and reduce uncertainty. 
Bayesian optimization (BO) and Bayesian experimental design (BED) provide powerful but largely separate treatments of goal-directed optimization and information-seeking experimentation, leaving limited guidance for hybrid settings in which learning and optimization are intrinsically coupled. 
We propose \textbf{Pragmatic Curiosity (PraC)}, a unified framework for hybrid learning and optimization via active inference. 
PraC evaluates candidate queries by trading information gain about a task-relevant latent symbol against an expected regret-based potential over outcomes. 
This formulation foregrounds three operational design choices: which latent quantity should be clarified, how task value is encoded as regret, and how strongly information gain should be exchanged against pragmatic regret. 
We instantiate PraC across three regimes of increasing complexity: decision-oriented monitoring with fixed global symbols and known downstream losses, targeted active search with induced local symbols and evolving coverage goals, and composite Bayesian optimization with hierarchical regret learning under unknown preferences. 
Across these regimes, PraC reduces downstream decision risk, improves coverage of critical outcome regions, and jointly learns predictive and preference structures without relying on task-specific staging rules. 
\end{abstract}

\section{Introduction}

Sequential decision-making under uncertainty often requires an agent to decide what to learn and what to do at the same time. 
Two classical paradigms have addressed these aspects from different directions. 
Bayesian optimization (BO) is primarily \textit{goal-directed}: it selects queries in order to identify high-performing inputs of an unknown objective function \citep{Mockus1975OnExtremum,Jones1998EfficientFunctions,Srinivas2009GaussianDesign}. 
Bayesian experimental design (BED) is primarily \textit{information-seeking}: it selects experiments in order to reduce uncertainty about latent quantities of interest \citep{Chaloner1995BayesianReview}. 
Although both can be viewed as instances of adaptive sampling \citep{di2023active}, they are usually developed under different directives and therefore admit few directly transferable design principles \citep{hvarfner2025informed}.

This separation leaves limited guidance for problems in which the relevant uncertainty, the downstream objective, and the sampling policy must be designed jointly.
In environmental monitoring \citep{KonakovicLukovic2020Diversity-guidedEvaluations}, a robot may need to learn the hidden state of a chemical plume while deciding where to dispatch a response team. 
In failure discovery \citep{Ramanagopal2018FailingCars, parashar2025cost}, an evaluator may need to explore uncertain scenarios while prioritizing outcomes that cover critical failure regions. 
In preference-guided design \citep{gonzalez2025implementation, coelho2025composite}, an optimizer may need to learn not only how actions produce outcomes, but also how those outcomes should be valued. 
In all these settings, an agent must act before it fully knows, and it only learns by acting. 
The bottleneck is therefore not learning or optimization alone, but the \emph{meta-decision} that connects them: deciding what knowledge is sufficient for action, and what actions are most informative for future decisions.

Recent work has begun to address this coupling by choosing between specialized tools and accommodating problem-specific adaptation: leveraging information-gain criteria to enhance optimization, or vice versa. 
Information-directed sampling (IDS) connects information and regret in online optimization and bandit settings \citep{Russo2018LearningSampling}. 
Self-correcting Bayesian optimization introduces a statistical distance-based active learning criterion into the BO loop to improve model learning during optimization \citep{Hvarfner2023Self-CorrectingLearningb}. 
Expected predictive information gain (EPIG) focuses active learning on predictive information that is relevant to a target distribution \citep{Smith2023Prediction-OrientedLearning}. 
Preference exploration for BO learns preference models over multi-objective outcomes, but typically relies on stage-wise choices between experimentation and preference learning \citep{Lin2022PreferenceOutcomes}. 
These methods highlight growing synergy between learning and optimization, but they often remain task-specific, and rarely generalize across problem categories.

In this paper, we propose \textbf{Pragmatic Curiosity (PraC)}, a unified framework for hybrid learning and optimization via active inference (AIF) \citep{Friston2010TheTheory, Friston2017ActiveTheory}. 
Originally developed in computational neuroscience, AIF prescribes action selection by minimizing \textit{expected free energy} (EFE), a single objective that couples an \textit{epistemic} drive for information gain with a \textit{pragmatic} drive toward preferred outcomes. 
PraC translates this conceptual principle into a practical acquisition-design template: a candidate query is evaluated by both what it is expected to reveal about a task-relevant latent quantity and how its predicted outcome is expected to affect downstream action. 
In this sense, PraC formalizes \textit{pragmatic curiosity}: curiosity is not the indiscriminate gathering of information, but the search for knowledge insofar as that knowledge can improve action pragmatically.

PraC not only serves as a unifying umbrella principle that reinterprets many classical acquisition rules in BO and BED, but also foregrounds a \textit{triad} of meta-decisions that govern any hybrid decision-making problems: \textit{representation}, \textit{evaluation}, and \textit{regulation}.
In PraC, this triad is conceptualized as \textit{symbols}, \textit{regrets}, and \textit{curiosity}. 
Symbols specify the latent quantities whose clarification can change what the agent ought to do; regrets provide a modest objective by quantifying task-relevant shortfall; and curiosity sets the exchange rate between epistemic clarification and pragmatic sacrifice. 
This triad turns PraC from an abstract decision paradigm into a concrete decision rule: it specifies what is worth knowing, what is worth pursuing or avoiding, and how strongly unresolved uncertainty should influence action. 
It also elevates curiosity from a static exploration heuristic into a regulated gain in the knowledge--action loop, motivating a feedforward--feedback scheduler that adapts curiosity as belief and action co-evolve.

We instantiate PraC across three regimes of increasing complexity. 
First, in \textit{decision-oriented monitoring}, PraC uses fixed global symbols and known downstream goals to guide sensing toward environmental distinctions that matter for response decisions, reducing the downstream decision risk by up to 39\% compared with greedy sampling. 
Second, in \textit{targeted active search}, PraC uses induced local symbols and evolving coverage goals to coordinate exploration in input space with discovery in outcome space, discovering 9.75 previously unseen active drug-scaffold families on average, compared with 2.00 under greedy search. 
Third, in \textit{composite Bayesian optimization with unknown preferences}, PraC uses hierarchical structures to jointly learn how actions produce outcomes and how those outcomes should be valued, automating vehicle safety design, bioprocess optimization, and energy allocation without relying on task-specific staging rules.
Together, these results demonstrate how a common Bayesian decision principle adapts to known, evolving, and unknown objectives, supporting autonomous monitoring, sample-efficient discovery, and complex engineering design under limited information and resources.

Our contributions are as follows. 
\begin{itemize}[leftmargin=10pt, nosep]
    \item \textbf{A unified acquisition principle.} 
    We derive PraC as an active-inference acquisition principle that couples task-relevant information gain with goal-directed pragmatic regret.

    \item \textbf{A triad of operational design axes.} 
    We formalize epistemic symbols, regret potentials, and curiosity coefficients as the core meta-decisions underlying hybrid decision-making problems.

    \item \textbf{Empirical validation across three regimes.} 
    We validate PraC in decision-oriented monitoring, targeted active search, and preference-guided composite optimization, spanning fixed global, induced local, and hierarchical symbols under both known and unknown goals.
\end{itemize}

\section{Preliminaries}

\subsection{Bayesian Optimization}
Given an unknown objective function $f:\mathcal{X} \to \mathbb{R}$, BO seeks to identify the input $x^{\ast}$ that maximizes the objective $f$ over an admissible set of queries $\mathcal{X}$, \ie, 
$
    x^{\ast} = \arg \max_{x \in \mathcal{X}} f(x).
$
Let the available information regarding the objective function $f$ be stored in the dataset $\mathcal{D}_{t}:=\{(x_{1},y_{1}), \dots, (x_{t},y_{t})\}$, where $y_{t} \sim \mathcal{N}(f(x_{t}),\sigma^{2}(x_{t}))$ is the noisy observation of the objective function by assuming the noise follows a zero-mean normal distribution with a standard deviation $\sigma(x)$. 
BO relies on a \textit{surrogate model} $p(f(x)|\mathcal{D}_{t})$ that provides a probabilistic representation of the objective $f$, and uses this information to compute an \textit{acquisition function} $\alpha:\mathcal{X} \to \mathbb{R}$ to drive the selection of the most promising sample to query.
Many acquisition functions have been proposed, including \textit{probability of improvement} \citep{Mockus1975OnExtremum}, \textit{expected improvement} \citep{Jones1998EfficientFunctions}, \textit{upper confidence bound}, and various \textit{entropy search} methods \citep{Hennig2012EntropyOptimization, Hernandez-Lobato2014PredictiveFunctions, Wang2017Max-valueOptimization, Hvarfner2022JointOptimization, Neiswanger2021BayesianInformation},
as well as practical approaches to optimize them \citep{Wilson2018MaximizingOptimization}.

\subsection{Bayesian Experimental Design}

Rather than optimizing an objective function $f(x)$, the purpose of BED is to sequentially select a set of experimental designs $x \in \mathcal{X}$ and gather outcomes $y$, to maximize the amount of information obtained about certain \textit{parameters of interest}, denoted as $\theta$. 
The parameters $\theta$ can correspond to some explicit model parameters, or any implicitly defined quantity of interest (\eg, the optimum of a function, the output of an algorithm, or future downstream predictions).
BED finds the next experimental design $x_{t+1}$ by maximizing the \textit{expected information gain} (EIG) \citep{Chaloner1995BayesianReview} that a potential experimental outcome $y_{t+1}$ can provide about $\theta$, measured in terms of expected entropy reduction of the posterior distribution of $\theta$:
\begin{equation}
\begin{split}
    \text{EIG}(x\mid \mathcal{D}_{t})&=H[p(\theta|\mathcal{D}_{t})] - \mathbb{E}_{p(y\mid x, \mathcal{D}_{t})}[H[p(\theta|\mathcal{D}_{t}\cup\{(x,y)\})]]
    =I(\theta;(x,y)\mid \mathcal{D}_{t}),
\end{split}
\end{equation}
where $H(\cdot)$ and $I(\cdot)$ denote the entropy and mutual information, respectively.

\section{Pragmatic Curiosity: A Framework that Synthesizes Learning and Optimization}

Acquisition strategies in BO typically pursue \textit{goal-directed} behavior, whereas those in BED emphasize \textit{information-seeking} behavior. 
In this section, we synthesize these seemingly competing imperatives through the lens of active inference (AIF), yielding a unified framework for hybrid learning and optimization.

\subsection{From Passive Inference to Active Inference}

Let agent maintain a probabilistic surrogate model $q(\cdot)$ that captures the relationship between a decision variable $x$, an outcome $y$, and a collection of latent quantities of interest $s$, and factorizes as
\begin{equation}
    q(x,y,s)\coloneq p(x,y\mid s)\,q(s),
\end{equation}
where $q(s)$ is the surrogate belief over $s$. 
We use $q(\cdot)$ to distinguish the agent's internal surrogate model from the underlying, generally inaccessible, true data-generating model $p(\cdot)$. 

The deviation between an observed outcome $y$ and those expected by $p(\cdot)$ is measured by a \textit{self-information}, or \textit{surprisal}, $-\log p(y)$. 
Intuitively, surprisal quantifies how unexpected an outcome $y$ is under distribution $p(y)$: less expected outcomes incur larger surprisal.
The surprisal associated with an outcome $y$ satisfies
\begin{equation}\label{eqn:VFE}
\begin{split}
    -\log p(y)
    &= -\log \int \frac{p(y,s)q(s)}{q(s)}\,ds 
    = -\log \mathbb{E}_{q(s)}\!\left[\frac{p(y,s)}{q(s)}\right]
    \le -\mathbb{E}_{q(s)}\!\left[\log \frac{p(y,s)}{q(s)}\right]
    = F,
\end{split}
\end{equation}
where the inequality follows from \textit{Jensen's inequality}.

The right-hand side of \eqref{eqn:VFE} is called the \textit{variational free energy} (VFE), resembling the \textit{Helmholtz free energy} in physics.
It upper bounds surprisal and therefore provides a tractable surrogate objective for inference. 
In machine learning, the sign of VFE is often reversed, yielding the evidence lower bound (ELBO), whose maximization is a standard variational learning principle \citep{Titsias2009VariationalProcesses}.

A central shift from passive inference to active inference is to account not only for realized outcomes, but also for future outcomes that have not yet been observed. We thus consider the expected surprisal of future outcomes under a predictive joint distribution $q(y, s \mid x)$:
\begin{equation}\label{eqn:EFE}
\begin{split}
    -&\mathbb{E}_{q(y,s\mid x)} \left[ \log \frac{p(y,s\mid x)}{q(s\mid x)} \right] 
    =-\mathbb{E}_{q(y,s\mid x)} \left[ \log \frac{p(s\mid x,y) p(y\mid x)}{q(s\mid x)} \right] \\
    %=-&\mathbb{E}_{q(y,s\mid x)} \left[ \log p(s\mid x,y) - \log q(s\mid x) \right] 
    %- \mathbb{E}_{q(y,s\mid x)} \log p(y\mid x) \\
    &=-\mathbb{E}_{q(y,s\mid x)} \left[ \log p(s\mid x,y) - \log q(s\mid x) \right] 
    -\mathbb{E}_{q(y\mid x)} \log p(y\mid x) = G,
\end{split}
\end{equation}
where the right-hand side of \eqref{eqn:EFE} is denoted as the \textit{expected free energy} (EFE). 
Active inference prescribes action selection by minimizing this quantity.

\subsection{A Change of Reference}

In its original form, the EFE in \eqref{eqn:EFE} is not yet an actionable acquisition principle, because it still refers to an \textit{objective, environment-referenced} model $p(\cdot)$, which is generally inaccessible to the agent. 
To make it a functional principle of action, we propose a change of reference from the environment to the agent, yielding an acquisition principle referred solely to the \textit{subjective, agent-referenced} model $q(\cdot)$.
This is achieved by introducing two \textit{subjectivizing shifts} in \eqref{eqn:EFE}:
\begin{equation}
    G
    =
    \underbrace{-\mathbb{E}_{q(y,s\mid x)}\left[\log p(s\mid x,y)-\log q(s\mid x)\right]}_{\text{Term 1}}
    \;
    \underbrace{-\mathbb{E}_{q(y\mid x)}\log p(y\mid x)}_{\text{Term 2}}.
\end{equation}

\paragraph{Belief internalization.}
For Term 1, add and subtract $\log q(s\mid x,y)$ inside the expectation:
\begin{equation}
\begin{split}
    \text{Term 1} &= -\mathbb{E}_{q(y,s\mid x)} \left[ \log p(s\mid x,y) - \log q(s\mid x,y) + \log q(s\mid x,y) - \log q(s\mid x) \right] \\
    &= -\mathbb{E}_{q(y\mid x)} \mathbb{E}_{q(s\mid x,y)} \left[ \log p(s\mid x,y) - \log q(s\mid x,y) + \log q(s\mid x,y) - \log q(s\mid x) \right] \\
    &= -\mathbb{E}_{q(y\mid x)} \left[ \mathbb{E}_{q(s\mid x,y)} \left[ \log p(s\mid x,y) - \log q(s\mid x,y) \right] + \mathbb{E}_{q(s\mid x,y)} \left[ \log q(s\mid x,y) - \log q(s\mid x) \right] \right] \\
    &= -\mathbb{E}_{q(y\mid x)} \left[\underbrace{- D_{\text{KL}}(q(s\mid x,y)\Vert p(s\mid x,y))}_{\text{Term 3}} + \underbrace{D_{\text{KL}}(q(s\mid x,y)\Vert q(s\mid x))}_{\text{Term 4}}  \right].
\end{split}
\end{equation}

Since the surrogate belief over $s$ is not updated until an observation is obtained, it does not depend on $x$ alone; hence $q(s\mid x)=q(s)$. 
Therefore, Term 4 becomes $D_{\text{KL}}(q(s\mid x,y)\Vert q(s))$, which is precisely the \textit{epistemic value}: the information gained about the latent quantity $s$ from a prospective observation $(x,y)$.

Term 3 quantifies an inaccessible discrepancy between subjective and omniscient belief, which can be treated as a \textit{cognitive friction}. 
We thus consider an internal closure of posterior belief within the agent’s current epistemic model by idealizing this friction to be zero: the agent cannot optimize with respect to an inaccessible ground truth; the best it can access is its current surrogate belief constructed from the historical data $\mathcal{D}_t$, namely $q(\cdot)=p(\cdot\mid\mathcal{D}_t)$. This yields the first subjectivizing shift.

\paragraph{Desire specification.}
The objective predictive distribution $p(y\mid x)$ in Term 2 describes which outcomes would occur, but it does not specify which outcomes should be pursued.
We therefore introduce a second subjectivizing shift by replacing this \textit{descriptive, prediction} distribution with a \textit{normative, preference} distribution $p_{\text{pref}}(y)$. 
This converts Term 2 into $- \mathbb{E}_{q(y\mid x)} \log p_{\text{pref}}(y)$, which evaluates preferred outcomes under the surrogate predictive model induced by choosing $x$. 
Intuitively, the agent now asks: if I take action $x$, what outcomes am I likely to see, and how desirable are those outcomes?

To define a preference distribution from a task-level notion of desirability, we introduce a Boltzmann operator $\mathcal{B}_\beta$ that maps a nonnegative potential energy function $h:\mathcal{Z}\to\mathbb{R}_{\ge 0}$ into a probability distribution:
\begin{equation}
    (\mathcal{B}_{\beta}[h])(z)\coloneq\frac{e^{-h(z)/\beta}}{\int_{\mathcal{Z}}e^{-h(z)/\beta}dz}.
\end{equation}
%This construction is inspired by the Boltzmann (or Gibbs) distribution in statistical mechanics. 
The parameter $\beta$ plays the role of temperature: larger $\beta$ yields a flatter, higher-entropy distribution, whereas smaller $\beta$ concentrates more sharply around the minimizers of $h$.

Applying this operator to a possibly time-varying potential energy function $h(y\mid\mathcal{D}_t)$ gives a preference distribution over outcomes, and hence
\begin{equation}
\begin{split}
    \text{Term 2} &= 1/\beta \cdot \mathbb{E}_{q(y\mid x, \mathcal{D}_{t})}[h(y\mid \mathcal{D}_{t})] + Z,
\end{split}
\end{equation}
where $q(y\mid x,\mathcal{D}_{t})$ is the predictive distribution of a surrogate model constructed from the historical data $\mathcal{D}_{t}$, and $Z=\log \int_{\mathcal{Y}} e^{-h(y\mid \mathcal{D}_{t})/\beta} dy$ is a normalization constant independent of $x$.

\subsection{A Unified Acquisition Principle}

We now propose an acquisition principle based on the subjectivized EFE:
\begin{equation}
    x_{t+1}=\arg \max_{x\in \mathcal{X}}\{ \beta_{t} \underbrace{I(s;(x,y)\mid \mathcal{D}_{t})}_{\text{epistemic}} - \underbrace{\mathbb{E}_{q(y\mid x, \mathcal{D}_{t})}[h(y\mid \mathcal{D}_{t})]\}}_{\text{pragmatic}}.
\end{equation}
Here, the conditional mutual information $I(s;(x,y)\mid\mathcal{D}_t)$ captures the \textit{epistemic} drive to reduce uncertainty about latent quantities of interest; the expected potential $\mathbb{E}[h(y\mid\mathcal{D}_t)]$ captures the \textit{pragmatic} drive to avoid undesirable outcomes; and $\beta_t\ge 0$ controls the trade-off between those two.

By construction, this decision paradigm balances \textit{information-seeking} and \textit{goal-directed} behavior within a single objective. 
Rather than treating exploration and exploitation as competing heuristics, it expresses them as two inseparable aspects of one decision criterion: the agent seeks information insofar as it improves outcome-sensitive action. In other words, it demonstrates a \textbf{pragmatic curiosity (PraC)}.

PraC serves as a unifying umbrella principle for both learning-oriented and optimization-oriented acquisition strategies. 
By varying the epistemic target $s$ and the pragmatic potential $h_t(y)$, many classical acquisition rules from BO and BED can be recovered as special cases or close approximations, as summarized in Table~\ref{tab: unified acquisition strategies}.

More importantly, this unifying principle opens the door to acquisition strategy design for a broader class of \textit{hybrid} learning-and-optimization problems that go beyond classical BO and BED. 
Such problems are typically governed by three coupled meta-decisions: (i) \textit{what the agent seeks to know}, (ii) \textit{what it seeks to achieve}, and (iii) \textit{how it trades off those two}. 
In the next section, we formalize these three design axes through \textit{symbols}, \textit{regrets}, and \textit{curiosity}.

\begin{table*}[h]
    \centering
    \caption{A unified view of different acquisition strategies in BO and BED, where $x^{\ast}$, $y^{\ast}$ represent the true optimal solution and value, respectively, and $\hat{y}$ is the best value observed in $\mathcal{D}_{t}$.}
    \label{tab: unified acquisition strategies}
    \begin{tblr}{colspec={|X[6.2,l]|X[2.2,c]|X[0.7,c]|X[1.0,c]|}, rowspec={Q[m]Q[m]Q[m]Q[m]Q[m]Q[m]Q[m]Q[m]Q[m]Q[m]}}
    \hline 
    \SetCell[r=2]{l} Acquisition Strategy & \SetCell[r=2]{l} Acquisition Function & \SetCell[c=2]{c} Pragmatic Curiosity \\
    \hline 
    & & $s$ & $h_t(y)$ \\
    \hline 
    Expected Information Gain \citep{Chaloner1995BayesianReview} & $I(\theta;(x,y)\mid \mathcal{D}_{t})$ & $\theta$ & - \\
    \hline 
    Entropy Search  \citep{Hennig2012EntropyOptimization} & $I(x^{*};(x,y)\mid \mathcal{D}_{t})$ & $x^{*}$ & - \\
    \hline
    Max-value Entropy Search  \citep{Wang2017Max-valueOptimization} & $I(y^{*};(x,y)\mid \mathcal{D}_{t})$ & $y^{*}$ & - \\
    \hline
    Joint Entropy Search  \citep{Hvarfner2022JointOptimization} & $I((x^{*}, y^{*});(x,y)\mid \mathcal{D}_{t})$ & $(x^{*}, y^{*})$ & - \\
    \hline
    Bayesian Algorithm Execution  \citep{Neiswanger2021BayesianInformation} & $I(\mathcal{O}_{\mathcal{A}}(f);(x,y)\mid \mathcal{D}_{t})$ & $\mathcal{O}_{\mathcal{A}}(f)$ & - \\
    \hline
    GP-Upper Confidence Bound \footnotemark[1] \citep{Srinivas2009GaussianDesign}  & $\mu_{t}(x)+\beta^{1/2}\sigma_{t}(x)$ & $f_{\mathcal{X}}$ & $-y$ \\
    \hline 
    Probability of Improvement  \citep{Mockus1975OnExtremum} & $p(y\ge\hat{y})$ & - & $-\mathbb{I}(y \ge \hat{y})$ \\
    \hline
    Expected Improvement  \citep{Jones1998EfficientFunctions} & $\mathbb{E}([y-\hat{y}]_{+})$ & - & $-[y-\hat{y}]_{+}$ \\
    \hline
\end{tblr}
\end{table*}

\footnotetext[1]{Strictly speaking, the correspondence for GP-UCB is approximate rather than exact, because GP-UCB uses a first-order standard-deviation term rather than a second-order variance term. Detailed derivations are provided in Appendix~\ref{appx:GP-UCB}. Still, the connection reveals the information-theoretic structure underlying this heuristic form.}

\section{The Triad of Meta-Decisions: Symbols, Regrets, and Curiosity}

One central contribution of PraC is that it foregrounds three meta-decisions underlying hybrid learning and optimization: what latent quantity should be clarified, encoded by $s$; how outcomes should be evaluated, encoded by $h_t$; and how strongly unresolved uncertainty should influence action, encoded by $\beta_t$.

These choices are the operational levers through which an abstract decision paradigm becomes a concrete decision rule. 
This section develops them both conceptually and operationally: \textit{symbols} specify the latent quantities whose clarification can change action; \textit{regrets} provide an outcome-level landscape of task-relevant shortfall; and \textit{curiosity} regulates the exchange rate between clarifying knowledge and pursuing goals.

\subsection{Symbols and Models}\label{subsec:symbols}

A model is a way of forgetting. To describe the world, we leave things out. Irrelevant details vanish, and what remains are the distinctions whose clarification can change what we should do.
In PraC, we call these distinctions \textit{symbols}. 
Different modeling regimes induce different forms of symbols.

\paragraph{Parametric models.}
In parametric models, symbols are typically \textit{fixed global symbols}: the latent state is represented by a finite-dimensional parameter whose components retain persistent semantic identities across decisions. 
Examples include unknown physical coefficients, latent modes, transition parameters, or class labels. 
In this regime, the epistemic objective is to reduce uncertainty about these stable hidden meanings.

\paragraph{Non-parametric models.}
In non-parametric models, there may be no finite set of globally privileged coordinates. 
Instead, the relevant symbols are often \textit{induced local symbols}, created by the contemplated observations themselves. 
This is formalized by the following property.

\begin{lemma}[\textbf{Representation Compression}]\label{lemma:mutual info}
Let $\mathbf{X}\subseteq \mathcal{X}$ be a subset of inputs, and let $f_{\mathbf{X}}$ denote the corresponding function values. 
Given a historical dataset $\mathcal{D}_t$ and new measurements $\mathbf{Y}$ observed at $\mathbf{X}$, then
$I(f_{\mathcal{X}};(\mathbf{X},\mathbf{Y})\mid\mathcal{D}_t)=I(f_{\mathbf{X}};\mathbf{Y}\mid\mathcal{D}_t),$
for any (finite or infinite) set $\mathcal{X}$.
\end{lemma}
\begin{proof}
    See Appendix \ref{appendix: mutual info}.
\end{proof}

Lemma~\ref{lemma:mutual info} shows that, under a non-parametric model such as a Gaussian process, the information gained from observing $\mathbf{Y}$ at $\mathbf{X}$ is mediated entirely through the function values $f_{\mathbf{X}}$ that directly generate those observations. 
These local evaluations play a role analogous to parameters in parametric models or latent states in partially observed dynamical systems, but they need not belong to a fixed global representation.

This compression reveals both the power and the cost of non-parametric representation. 
The power is parsimony: the agent need not reason about every component of an infinite-dimensional latent object when only a local slice can influence the prospective observation. 
The cost is semantic instability: because the symbol is query-dependent, its meaning is induced by the decision context rather than fixed in advance. 

This also suggests a broader role for PraC. 
PraC asks not merely how expressive a model is, but whether the symbol used by the acquisition is sufficient for action evaluation, minimal enough to avoid epistemic redundancy, and structured enough to support interpretable decision-making.
Thus beyond selecting actions, PraC can also guide the construction of representations. 
When the relevant structure is unclear, induced local symbols can reveal which distinctions are repeatedly useful for decision. 
Over time, such recurring local distinctions may be consolidated into more stable global symbols. 
In this sense, PraC does not only act within a representation; it also provides a criterion for which representations deserve to persist.

\subsection{Regrets and Goals}\label{subsec:regrets}

A great paradox of life is that we often cannot fully know our goal until we arrive there. What we desire for the future is shaped by what we regret in the past.
In PraC, we treat \textit{regret} as a surrogate representation of desire: a visible outcome cast from a vague goal. 
The goal itself may be abstract, incomplete, or only gradually revealed, but the regret it induces makes that goal evaluable through the consequences of action.

Rather than requiring a complete specification of how every possible outcome should be valued, regret asks a more modest, reference-relative question: \textit{how does this outcome fall short of what is desired?} It preserves distinctions that matter for pursuing a goal while allowing distinctions irrelevant to that goal to disappear. 
We formalize this notion as follows.

\begin{definition}[\textbf{Regret Function}]\label{def:ture regret}
Let $\mathcal{Y}$ denote the outcome space, and let $\xi$ denote a desired reference, such as an optimal outcome, a target set, a safety specification, or a preference criterion. 
A regret function
$$
r(\cdot;\xi):\mathcal{Y}\to\mathbb{R}_{\ge 0}
$$
quantifies how an outcome $y$ falls short of the desired reference $\xi$. It obtains $0$ for the desired outcomes, \ie, $\inf_{y\in\mathcal{Y}} r(y;\xi)=0$, with smaller values indicating outcomes that are closer to what is desired.
\end{definition}

This abstraction recovers several important cases:
\begin{itemize}[leftmargin=10pt, nosep]
\item When $\xi=y^\ast$ is the global optimum of an objective function, $r(y;\xi)$ reduces to the conventional instantaneous regret in BO.
\item When no outcome-dependent pragmatic preference is imposed, as in pure BED, one may take $r(y;\xi)\equiv 0$, so that decision-making is driven entirely by epistemic value.
\item Intermediate choices recover hybrid objectives that jointly encode performance, safety, feasibility, or other task-specific desiderata.
\end{itemize}

The regret function $r(\cdot;\xi)$ expresses an omniscient structure of desire, but an agent often cannot possess this structure in a complete or directly usable form. To act, the agent therefore requires a subjective, decision-usable representation of regret under its current knowledge:

\begin{definition}[\textbf{Potential Energy Function}]
The potential energy function is a belief-conditioned surrogate regret
$$
h_t:\mathcal{Y}\to\mathbb{R},
\qquad
h_t(y):=h(y\mid\mathcal{D}_{t}),
$$
that represents the agent's current internal assessment of the undesirability of outcome $y$ based on the information available up to time $t$. 
\end{definition}

Thus, $h_t$ serves as a time-varying, belief-dependent approximation of the true regret structure, and is updated recursively as new data are acquired. 
This connects back to the desiderative subjectivizing shift introduced when deriving PraC: In parallel to belief internalization, where the inaccessible true posterior is replaced by the current surrogate belief, now the inaccessible true regret landscape is replaced by a belief-conditioned potential over outcomes. 
As data accumulate, both the predictive model and the potential energy evolve, so that the agent's internal representation of desire is refined together with its understanding of the world.
This aspect determines how $h_t$ should be constructed under different levels of knowledge about the goal.

\paragraph{Known goals.}
When the goal is known, $h_t$ can be designed to directly encode domain knowledge. 
%In this case, the designer already has access to a meaningful notion of what constitutes shortfall, so the potential energy can be specified explicitly. 
Examples include classical regret with respect to an optimum, distance to a target set, penalties for constraint violation, or weighted combinations of task performance and safety requirements. 
In this regime, the role of $h_t$ is to express a known desiderative structure in a form that can be evaluated under the predictive model.

\paragraph{Unknown goals.}
When the goal is unknown or only partially specified, the regret itself must be learned. 
In this regime, the pragmatic term should be endowed with additional structure, so that the agent can infer not only how the world behaves, but also what outcomes should count as desirable or undesirable. 
A natural construction is hierarchical: outcome desirability is mediated by additional latent variables or higher-level regret models. 
Given an unknown regret model $g:\mathcal{Y}\to\mathbb{R}_{\ge 0}$, one can define $h_t$ through an inner EFE over that regret model, yielding the nested acquisition
\begin{equation}\label{eqn:AF nested}
    \begin{split}
        \alpha(x\mid \mathcal{D}_{t}) = \beta_t &I(f_{\mathcal{X}};(x,y)\mid \mathcal{D}_{t}) 
        +\mathbb{E}_{q(y\mid x, \mathcal{D}_{t})}\left[\right.
        \gamma_t I(g_{\mathcal{Y}};(y,z)\mid \mathcal{D}_{t})-\mathbb{E}_{q(z\mid y,\mathcal{D}_{t})}[z]\left.\right],
    \end{split}
\end{equation}
where $\beta_t, \gamma_t \ge 0$, and $z \sim g(y)$ encodes the regret of outcomes. 

In this way, the agent does not optimize with respect to a fully specified reward function given in advance; rather, it learns a surrogate of regret jointly with the predictive model. 
This yields a richer and more flexible account of the interplay between learning and optimization, and enables multiple layers of inference to be composed within a single decision rule. 
For instance, as we will show later in \ref{subsec:cbo}, the model $g(y)$ can be learned through preference feedback. 

This point highlights a key advantage of the PraC paradigm. 
Unlike many mainstream approaches that presuppose a fully specified reward function, PraC requires only a current surrogate of what can already be recognized as desirable or regrettable, together with a mechanism for refining that surrogate as evidence accumulates. 
We do not need to wait until fully illuminated; it is enough to start by seeing a light in the darkness, and to let both understanding of the world and understanding of the goal sharpen together.

\subsection{The Degree of Curiosity}\label{subsec:curiosity}

Symbols specify what hidden distinctions are worth learning; regrets specify what desired outcomes are worth pursuing. The coefficient $\beta_t$ determines how much immediate pragmatic sacrifice the agent is willing to tolerate in order to refine its epistemic understanding. 
In this sense, $\beta_t$ is the exchange rate between clarifying knowledge and pursuing goals, which we denote as the degree of \textit{curiosity}.

Different regimes of $\beta_t$ produce different modes of behavior. 
When $\beta_t=0$, the agent treats its current symbols as fixed and acts myopically with respect to the present surrogate of regret. 
At the other extreme, a very large $\beta_t$ prioritizes information acquisition even when that information has weak pragmatic relevance. 
Neither extreme is satisfactory: seeking knowledge without goals degenerates into aimless information gathering, whereas pursuing goals without clarified meanings risks acting on an incomplete understanding of the world. 

This reveals a closed-loop structure between knowledge and action. 
The agent acts according to its current symbolic understanding and regret surrogate; the resulting observation then reshapes that understanding and changes which future actions become desirable. 
Thus, curiosity should not be viewed as a static exploration heuristic. 
Rather, it is a regulated gain in this knowledge--action loop: it determines how strongly unresolved uncertainty should influence action when the agent's current understanding may still be insufficient for reliable decision-making. 
This motivates designing $\beta_t$ as an adaptive controller for this knowledge--action loop. 

\paragraph{Feedforward scale.}
The feedforward component estimates the scale of curiosity required for epistemic value to be comparable with pragmatic value. 
Given a finite candidate set $\mathcal X_t^{\mathrm{cand}}$, let
$
\widehat I_t(x)\approx I(s;(x,y)\mid\mathcal D_{t}) 
$
estimates the epistemic information gain of querying $x$, and $\widehat H_t(x):=\mathbb E[h_t(y)\mid x,\mathcal D_{t}]$ estimates its expected pragmatic regret. 
Let the informative candidate set be
$
\mathcal X_t^{\mathrm{info}}
=
\{x\in\mathcal X_t^{\mathrm{cand}}:\widehat I_t(x)>\varepsilon_I\},
$
where $\varepsilon_I>0$. 
With a quantile level $\rho\in(0,1)$, the feedforward scale is given as
$$
\beta_t^{\mathrm{ff}}
=
\operatorname{Quantile}^{(\rho)}_{x\in\mathcal X_t^{\mathrm{info}}}
\frac{\widehat H_t(x)}
{\widehat I_t(x)+\varepsilon_I},
$$
which estimates the local exchange rate between information and regret: when useful information is scarce relative to pragmatic regret, the required curiosity scale increases; when information is readily available, a smaller scale is sufficient to make information gain compete with immediate pragmatic regret.

\paragraph{Feedback activation.}
The feedback component determines whether an epistemic pressure is still needed.
Let $U_t$ be a scalar uncertainty measure over the relevant symbol, such as posterior entropy, and let $U^\star$ denote the desired epistemic equilibrium, typically zero or an irreducible uncertainty floor. 
Define 
$$
\beta_t^{\mathrm{fb}}
=
k_{\beta}\operatorname{clip}
\left(
\frac{U_t-U^\star}{U_0-U^\star+\varepsilon_U}, 0, 1
\right),
$$
where $U_0$ is the initial uncertainty, $k_\beta$ is a feedback gain, and $\varepsilon_U>0$ prevents numerical instability. 

The combined schedule for curiosity coefficient is then
$$
\beta_t
=
\operatorname{clip}
\left(
\beta_t^{\mathrm{fb}}\beta_t^{\mathrm{ff}},
\beta_{\min},
\beta_{\max}
\right),
$$
with two components playing complementary roles: the feedforward term estimates how large curiosity must be for information gain to compete with pragmatic regret, while the feedback activation suppresses curiosity as the belief approaches the desired epistemic equilibrium. 
Thus, curiosity is high only when information is both decision-relevant and epistemically needed; it relaxes when the relevant uncertainty has been resolved.

\section{Experiments}

We evaluate PraC across three regimes that progressively expand what must be learned in the knowledge–action loop. The first regime fixes both the latent representation and downstream goal, isolating whether curiosity can regulate information acquisition for decision quality. The second replaces the fixed latent state with query-induced local symbols and an evolving coverage objective, testing whether epistemic exploration can be continually redirected by what remains useful to discover. The third makes the pragmatic objective itself latent, testing whether PraC can jointly learn how actions produce outcomes and how those outcomes should be valued. Together, the three regimes examine the three operational roles exposed by PraC: selecting task-relevant \textit{symbols}, encoding task value through \textit{regret}, and regulating their interaction through \textit{curiosity}. 

\subsection{Fixed Global Symbols with Known Goals}\label{subsec:decision_source_seeking}

We begin with the most transparent regime: the latent structure is fixed and the downstream goal is known. 
Here the symbol is a finite hypothesis $\theta$, and the task objective can be expressed through a downstream loss. 
This setting allows us to examine whether \textit{curiosity} can be elevated from an exploration heuristic to a monitored and controllable regulation of the decision process.

A typical application is \textbf{decision-oriented monitoring}, where information is valuable through the response it enables. The agent maintains a posterior over hypotheses $\theta$, selects a sensor query $x$, observes a measurement $y$, updates its belief, and evaluates the downstream decision loss. 
Let $a\in\mathcal A$ denote a downstream action and $L(a,\theta)$ the loss incurred under hypothesis $\theta$. 
Given the current surrogate posterior $q(\theta \mid \mathcal D_t)$, the Bayes risk is
\begin{equation}
    \mathrm{BR}(\mathcal D_t)
    \coloneq
    \min_{a\in\mathcal A}
    \mathbb E_{q(\theta \mid \mathcal D_t)}
    [L(a,\theta)].
\end{equation}
For a candidate query $x$, define the expected Bayes-risk reduction
\begin{equation}
    \Delta \mathrm{BR}_t(x)
    \coloneq
    \mathrm{BR}(\mathcal D_t)
    -
    \mathbb E_{q(y \mid x,\mathcal D_t)}
    \left[
        \mathrm{BR}(\mathcal D_t\cup\{(x,y)\})
    \right].
\end{equation}
The PraC acquisition therefore combines information about the latent hypothesis with the expected benefit of sensing for the downstream decision:
\begin{equation}
    \alpha_t(x)
    =
    \beta_t I(\theta;(x,y)\mid\mathcal D_t)
    +
    \Delta \mathrm{BR}_t(x).
    \label{eq:prac_decision_monitoring}
\end{equation}
%Thus, the pragmatic term is the expected improvement of the posterior decision state induced by the query-outcome pair $(x,y)$.

\paragraph{Tasks.}
We evaluate three 2D-monitoring problems with increasingly structured downstream consequences.
\emph{Source response localization} asks the agent to localize an uncertain source well enough to dispatch a response team. \emph{Consequence-weighted dispatch} introduces spatially asymmetric consequences, so errors near high-consequence regions matter more. \emph{Active source prioritization} asks the agent to choose which suspected sources to repair under a limited repair budget, with regret determined by the impact of active sources left untreated. Full task definitions are provided in Appendix~\ref{app:ss-tasks}.

\paragraph{Baselines.}
We compare \textsc{PraC-dynamic} with: (i) \textsc{Random}; (ii) \textsc{Greedy}, which maximizes expected Bayes-risk reduction; (iii) \textsc{EIG}, which maximizes information gain alone; (iv) \textsc{IDS} \citep{Russo2018LearningSampling}, which trades decision loss against information through an information-ratio criterion; and (v) \textsc{PraC-fixed}, which uses a constant curiosity coefficient. Implementation and scheduler settings are given in Appendix~\ref{app:ss-baselines}.

\paragraph{Metrics.}
Our primary metrics are \textit{posterior Bayes risk} and \textit{realized downstream loss}: response loss for the two dispatch tasks and missed-source risk for prioritization. The former quantifies expected loss under the current posterior; the latter evaluates the chosen action under the true hypothesis. Normalized Bayes-risk AUC, additional diagnostics, and paired comparisons are reported in Appendix~\ref{app:ss-metrics}--\ref{app:ss-detail-results}.

\paragraph{Results.}
Figure~\ref{fig:decision_source_seeking} shows how regulating curiosity changes the quality of decision-oriented sensing. In \textit{source response localization}, where the initial belief favors an incorrect region, dynamic PraC reduces Bayes risk substantially faster than all alternatives and achieves the lowest final posterior risk. The epistemic term prevents premature commitment to the dominant prior mode, while the pragmatic term keeps sensing directed toward information that improves the response decision. In \textit{consequence-weighted dispatch}, the same acquisition adapts to an asymmetric loss landscape and remains among the strongest decision-oriented policies, substantially improving over fixed PraC, EIG, and Random. Finally, \textit{source prioritization} requires sustained exploration of weak, low-prior but high-impact sources: here fixed PraC with stronger persistent curiosity achieves the lowest Bayes-risk AUC and missed-source risk. Across the three tasks, the useful curiosity regime therefore changes with the structure of the downstream decision, supporting the role of $\beta$ as an exchange rate between belief refinement and immediate pragmatic value.

The sensitivity analysis of the dynamic curiosity scheduler makes this regulatory role explicit (Appendix~\ref{app:ss-beta-sensitivity}). The best fixed $\beta$ shifts from $0.1$ in localization to $0$ in consequence-weighted dispatch and $5$ in source prioritization, whereas the dynamic scheduler adjusts epistemic pressure from the current information–regret balance. Across the tested settings, dynamic PraC places 44\% of settings within 5\% of the task optimum, compared with 20\% for fixed PraC, and reduces worst-case sensitivity in all three tasks. This supports the intended interpretation of curiosity as an adaptive control variable rather than a single globally tuned exploration bonus.

\begin{figure*}[t!]
    \centering
    \includegraphics[width=0.95\linewidth]{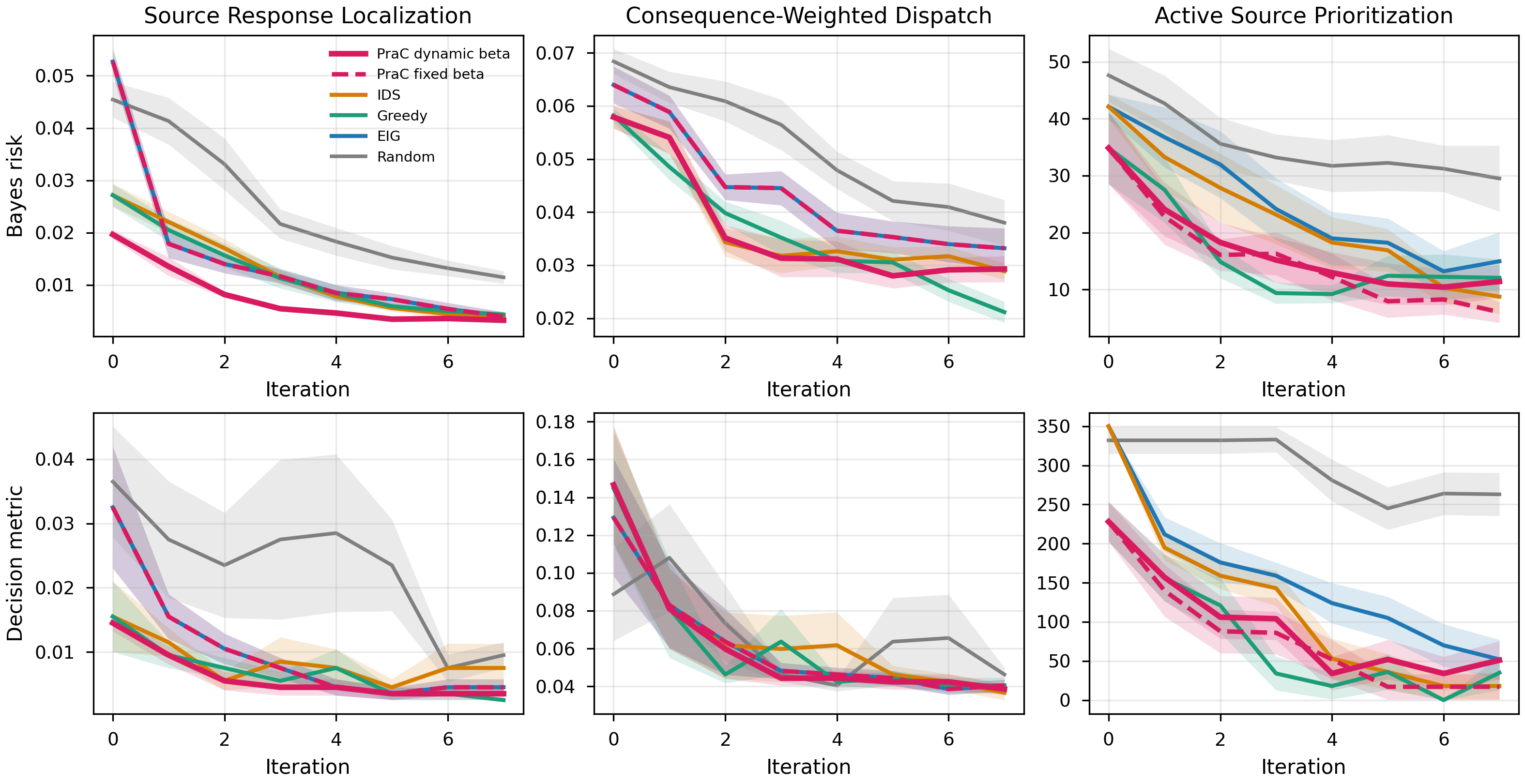}
    \caption{
    Performance in decision-oriented monitoring.
    Top row: posterior Bayes risk. 
    Bottom row: task-specific decision metric for realized downstream loss.  
    Shaded regions represent $\pm$ SEM over 20 seeds.
    }
    \label{fig:decision_source_seeking}
\end{figure*}

\subsection{Induced Local Symbols with Known Evolving Goals}

We next consider a regime in which the relevant latent structure is not a fixed global parameter, but a \emph{local symbol} induced by the contemplated observation. The goal remains known, but evolves with the outcomes already collected: once part of the target region has been covered, nearby outcomes become less valuable. This setting tests whether PraC can balance learning the local outcome structure with continually redirecting search toward what remains useful to discover.

A representative problem is \textbf{targeted active search} in multi-objective design, where the goal is to cover an important outcome region $\mathcal S$ rather than optimize a scalar reward. 
Following \citet{Malkomes2021BeyondDesign}, let $\delta$ define the resolution at which nearby outcomes are considered redundant, and define
$$
\mathbb C_\delta(y):=\{y':d(y,y')<\delta\},
\qquad
\mathbb C_\delta(Y):=\bigcup_{y\in Y}\mathbb C_\delta(y).
$$
A natural coverage-shortfall potential is
$$
h_t(y)
=
\mathrm{Vol}(\mathcal S)
-
\mathrm{Vol}\!\left(\mathbb C_\delta(Y\cup\{y\})\cap\mathcal S\right),
$$
which penalizes outcomes that add little new target coverage. 
Using the induced local symbol $f_x$ from Lemma~\ref{lemma:mutual info}, the PraC acquisition becomes
\begin{equation}
    \alpha_t(x)
    =
    \beta I(f_x;y\mid \mathcal D_t)
    +
    \mathbb E_{q(y\mid x,\mathcal D_t)}
    \left[
    \mathrm{Vol}\!\left(\mathbb C_\delta(Y\cup\{y\})\cap\mathcal S\right)
    \right].
\end{equation}

\paragraph{Tasks.}
We evaluate constrained active search in three domains that differ
in scientific object, representation, and oracle. \textit{Drug discovery} aim for active molecules across chemical scaffolds. \textit{Materials discovery} seeks high-$T_c$ superconductors across structural prototypes. \textit{Industrial inspection} targets discovering actionable defect mechanisms under partial supervision.
Full task definitions are provided in Appendix~\ref{app:tas-tasks}.

\paragraph{Baselines.}
We compare \textsc{PraC-dynamic} with (i) \textsc{Random}; (ii) \textsc{Greedy}, which selects the largest expected incremental target coverage; (iii) \textsc{EIG}, which selects the most informative query; and (iv) \textsc{IDS}, which trades pragmatic shortfall against information through an information-ratio criterion. PraC uses the dynamic curiosity scheduler with the default settings given in Appendix~\ref{app:tas-baselines}.

\paragraph{Metrics.}
Our primary metrics are \textit{family-balanced recall} and \textit{task-specific discovery counts}. Recall is averaged over initially unrepresented target families for drug and superconductors, and over actionable defect types, excluding calibration examples, for inspection. Discovery counts are hidden families for drug and superconductors, and actionable defects for inspection. Target hits, anchor/known-family recall, and paired comparisons are reported in Appendix~\ref{app:tas-metrics}--\ref{app:tas-detail-results}.

\paragraph{Results.}
Figure~\ref{fig:tas} shows that PraC converts epistemic exploration into broader and more sustained coverage of task-relevant outcome structure. In \textit{drug discovery}, PraC achieves the highest hidden-family macro recall and discovers \(9.75\) hidden scaffolds on average, substantially more than Greedy, EIG, and Random. IDS retrieves more individual active molecules, but PraC distributes its discoveries across more distinct scaffolds, directly reflecting the coverage objective. In \textit{superconductor discovery}, PraC attains both the highest mean hidden-prototype macro-recall AUC and the largest hidden-target yield, demonstrating that the same principle transfers from chemical scaffolds to rare material prototypes. In \textit{industrial inspection}, PraC obtains the strongest actionable-type and hidden-type coverage and the largest actionable discovery count. Across all three domains, the common pattern is that epistemic exploration expands the set of distinguishable outcome regions, while the evolving potential prevents search from repeatedly spending queries on distinctions that have already been sufficiently covered. This is the empirical role of induced local symbols in PraC: representation is created where new observations can still change what is useful to discover.

\begin{figure*}[t!]
    \centering
    \includegraphics[width=\linewidth]{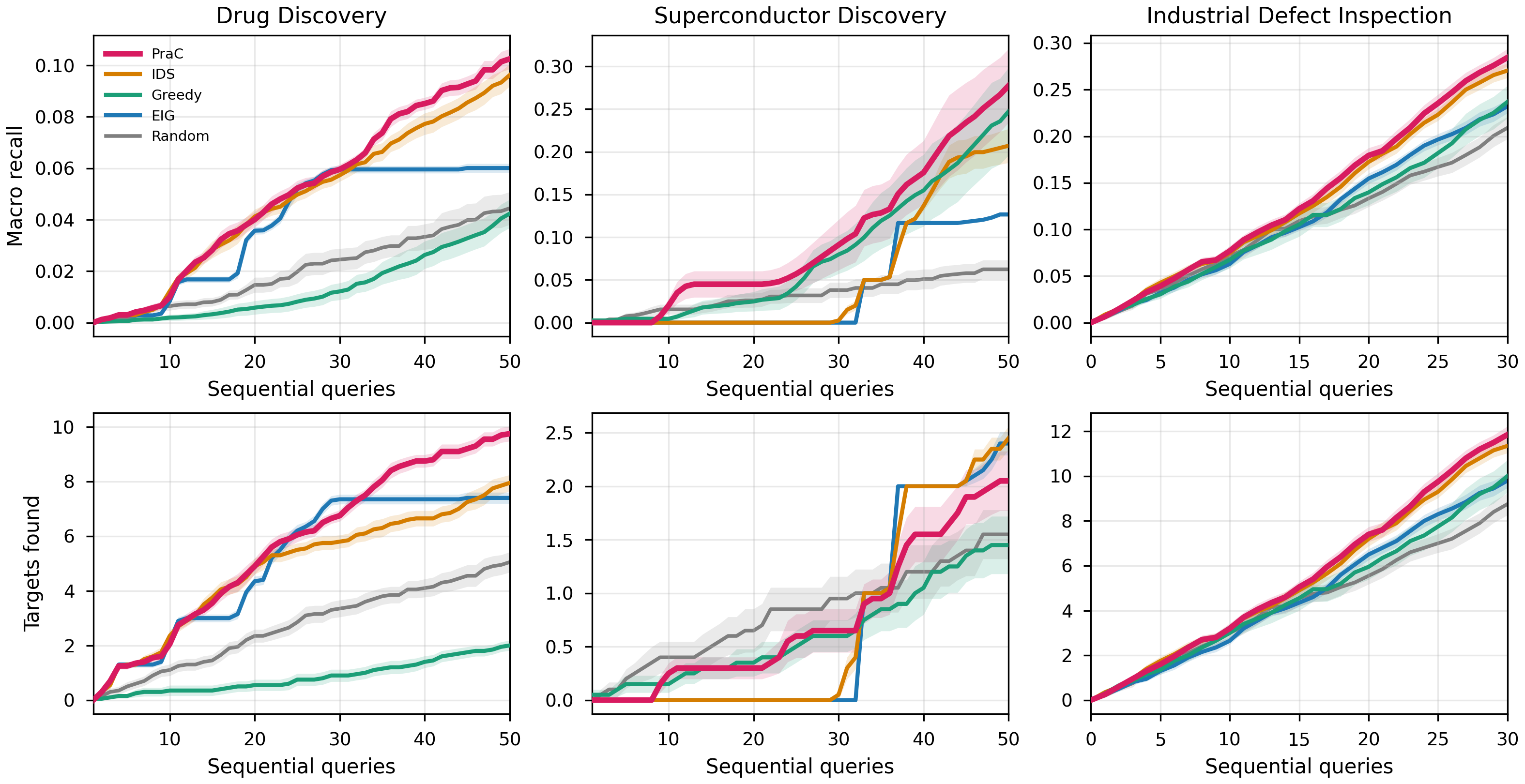}
    \caption{Performance in targeted active search. 
    Top row: Target family macro recall. 
    Bottom row: Targets found. 
    Shaded regions represent $\pm$ SEM over 20 seeds.}
    \label{fig:tas}
\end{figure*}

\subsection{Hierarchical Symbols with Unknown Goals}\label{subsec:cbo}

Finally, we consider the regime in which the goal itself is not fully specified. 
The agent must learn both how actions produce outcomes and how those outcomes should be valued. 
This setting emphasizes the role of \textit{hierarchical regret}: the pragmatic potential is not given directly, but inferred through an additional regret model. 
We use fixed curiosity weights to isolate the effect of this hierarchical design.

A representative problem is \textbf{composite Bayesian optimization} with unknown preferences. 
An action $x$ produces a multi-objective outcome $y=f(x)$, while an unknown regret model $g(y)$ determines a scalar evaluation over outcomes. 
Following \citet{Lin2022PreferenceOutcomes}, we consider learning $g$ from pairwise preference queries. 
For a pair $\tilde y=(y_1,y_2)$, let $l(\tilde y)\in\{1,2\}$ denote the preferred outcome. 
As in \citet{Chu2005PreferenceProcesses}, we use the probit likelihood
$$
\Pr(l(\tilde y)=1\mid g(y_1),g(y_2))
=
\Phi\!\left(\frac{g(y_2)-g(y_1)}{\sqrt{2}\lambda}\right),
$$
where lower $g(y)$ indicates lower regret and hence higher preference.

At each iteration, the method selects a candidate pair $\tilde x=(x_1,x_2)$, observes outcomes $\tilde y=(y_1,y_2)$, and receives a preference label $l$. 
The hierarchical PraC acquisition is thus
\begin{equation}
\begin{split}
    \alpha_t(\tilde x)
    =
    \beta_t I(f_{\mathcal X};(\tilde x,\tilde y)\mid \mathcal D_t)
    +
    \mathbb E_{q(\tilde y\mid \tilde x, \mathcal D_t)}
    \left[
        \gamma_t I(g_{\mathcal Y};(\tilde y,l)\mid \mathcal D_t)
        -
        \min_{y\in\{y_1,y_2\}}
        \mathbb E_{q(g(y)\mid y, \mathcal D_t)}[g(y)]
    \right],
\end{split}
\end{equation}
where the first term learns the outcome model $f$, the second learns the regret model $g$, and the final term exploits the currently predicted lowest-regret outcome.

\paragraph{Tasks.}
We evaluate three multi-objective design problems of increasing normative complexity: \emph{vehicle safety}, balancing crashworthiness and mass; \emph{penicillin production}, balancing yield, emissions, and processing time; and \emph{energy allocation}, balancing cost, voltage fairness, priority service, and resilience. The true preference underlying each task is hidden from the optimizer and used only to answer comparisons and evaluate discoveries. Full task definitions are provided in Appendix~\ref{app:cbo-tasks}.

\paragraph{Baselines.}
We compare \textsc{PraC-full} with (i) \textsc{Random}; (ii) \textsc{EIG-only}, retaining both information terms; (iii) \textsc{pragmatic-only PraC}, removing both information terms; and (iv) \textsc{PraC without outcome IG}, retaining preference exploration and pragmatic evaluation. These variants test the contributions of pragmatic evaluation, preference-directed exploration, and outcome-model exploration. We additionally compare against four BOPE-style variants \citep{Lin2022PreferenceOutcomes} that alternate preference elicitation and outcome optimization through scheduled stages. Implementation and BOPE schedules are given in Appendix~\ref{app:cbo-baselines}.

\paragraph{Metrics.}
Our main metric is the \textit{best true preference} attained among all collected outcomes, measuring whether the joint learning process discovers desirable designs. Normalized incumbent-preference AUC and paired statistical comparisons are reported in Appendix~\ref{app:cbo-metrics}--\ref{app:cbo-detail-results}.

\paragraph{Results.}
Fig.~\ref{fig:pbo} shows how the hierarchical terms become useful as uncertainty moves from the outcome model to the goal itself. In \textit{vehicle safety}, the pragmatic component already captures most of the attainable gain, and the hierarchical variants track this strong performance. In \textit{penicillin production}, full PraC achieves the highest mean final preference and anytime AUC, showing the benefit of jointly refining the outcome and preference models in a higher-dimensional process-design problem. The strongest separation appears in \textit{energy allocation}. Pragmatic selection alone and information acquisition alone both remain in low-preference regions, whereas adding preference information gain to pragmatic selection raises the mean final preference from \(-1.953\) to \(5.919\); full PraC reaches a comparable \(5.859\). The key empirical mechanism is therefore the coupling between learning what outcomes matter and acting on the current preference model. Outcome-model information gain provides an additional epistemic layer when needed, but the hierarchy allows each source of uncertainty to contribute through the same acquisition principle.

Figure~\ref{fig:cbo-bope} evaluates the same idea against BOPE-style staged preference learning and optimization. PraC reaches the highest mean final incumbent, \(7.0\), without introducing an explicit switching schedule; BOPE-I progresses faster early in the budget and therefore attains higher anytime AUC. The comparison demonstrates that preference learning and optimization need not be separated into manually scheduled phases: a single hierarchical acquisition can continuously exchange information value against current preference throughout the optimization process.

\begin{figure*}[t!]
    \centering
    \includegraphics[width=0.9\linewidth]{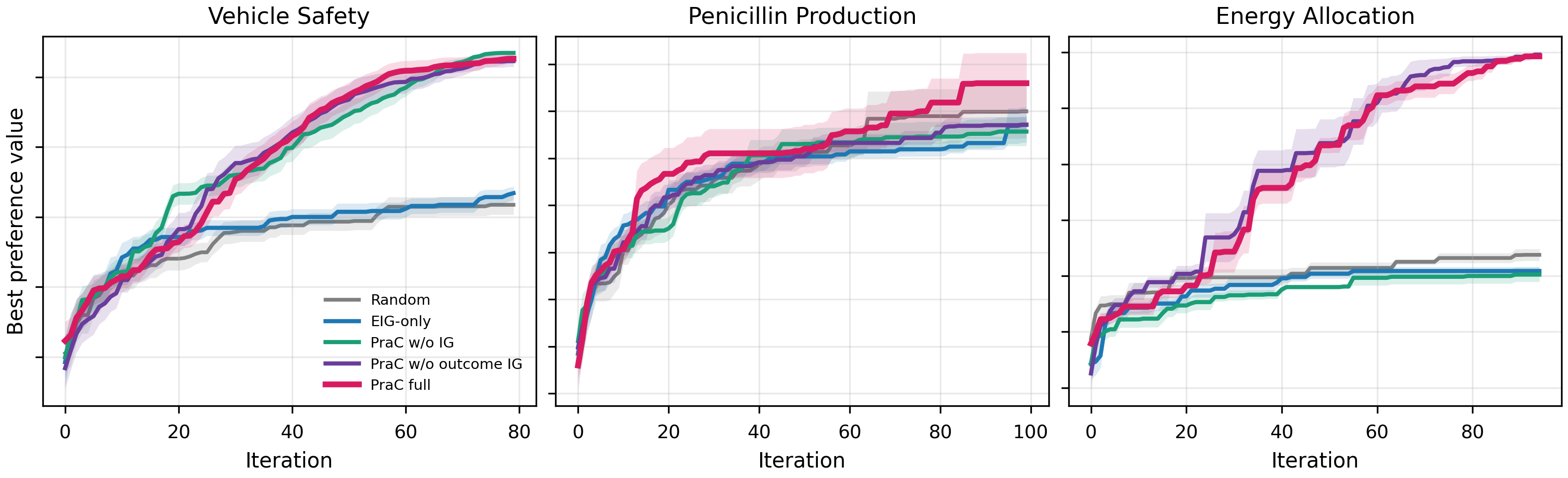}
    \caption{Performance in composite BO with unknown preferences. 
    Metric is the best preference value attained among all collected outcomes. 
    Shaded regions represent $\pm$ SEM over 20 seeds.}
    \label{fig:pbo}
\end{figure*}

\begin{figure*}[t]
\centering
\begin{minipage}[t]{0.45\textwidth}
\centering
\includegraphics[width=\linewidth]{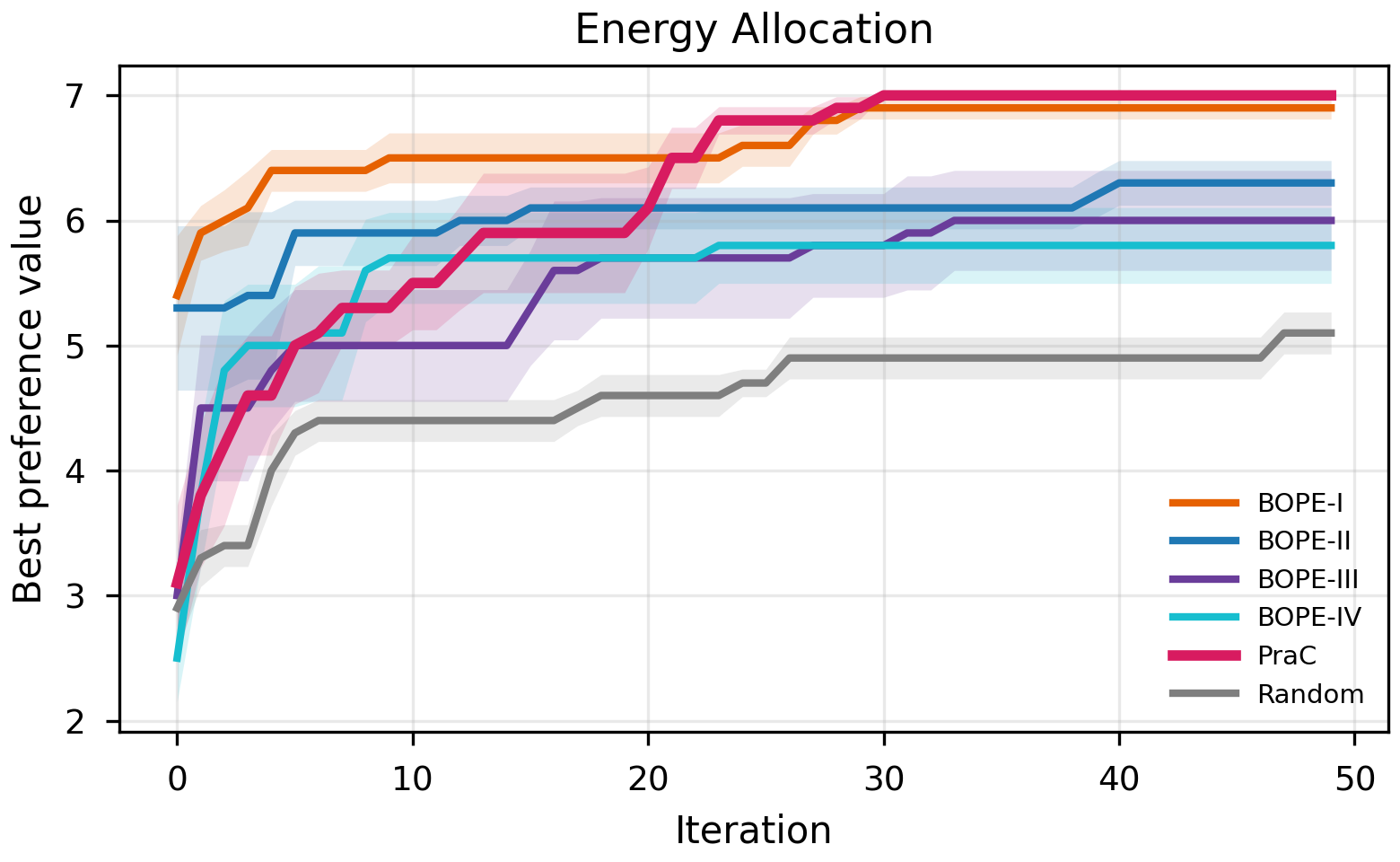}
\end{minipage}
\begin{minipage}[t]{0.45\textwidth}
\centering
\includegraphics[width=\linewidth]{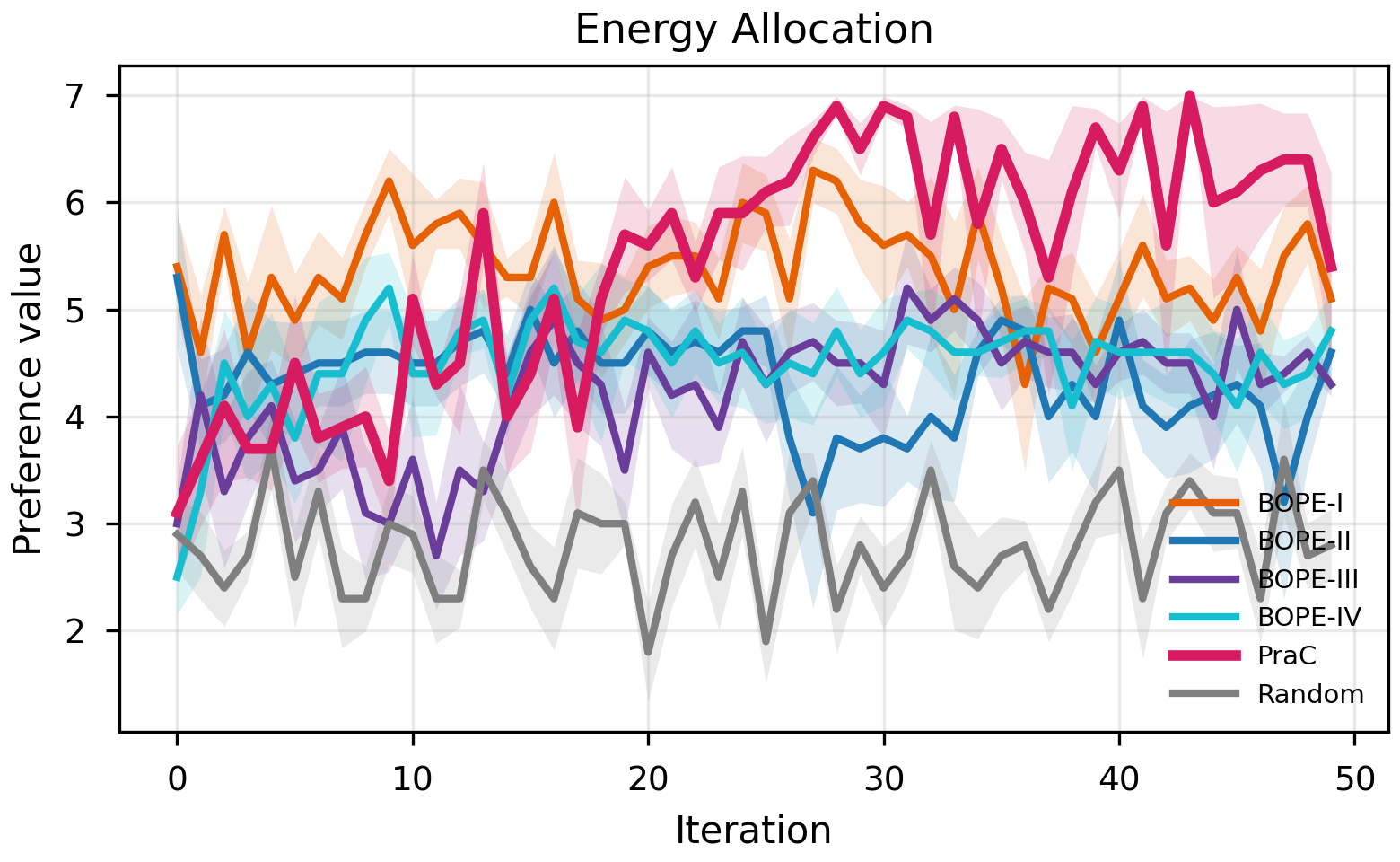}
\end{minipage}
\caption{Comparison with BOPE schedules in energy allocation. Left: incumbent preference.  Right: preference of the design selected at each round.  Shaded regions represent $\pm$ SEM over 20 seeds.}
\label{fig:cbo-bope}
\end{figure*}

\paragraph{Synthesis.} 
The three experimental regimes expose complementary roles of PraC. With fixed global symbols, curiosity regulates how much uncertainty must be resolved before acting. With induced local symbols, information acquisition expands the set of useful distinctions while the evolving potential redirects effort toward what remains uncovered. With hierarchical symbols, the same principle composes learning across levels, allowing the agent to learn both how actions produce outcomes and how those outcomes should be valued. The common empirical pattern is the ability to instantiate the same knowledge–action principle through different choices of \textit{what to know}, \textit{how to evaluate outcomes}, and \textit{how strongly unresolved knowledge should influence action}. 

\section{Conclusion}

This paper introduced \textbf{Pragmatic Curiosity (PraC)}, a unified framework for hybrid learning and optimization. The framework exposes three operational design choices: \textit{symbols}, the latent quantities whose clarification can change action; \textit{regrets}, the potential landscapes that encode task value or shortfall; and \textit{curiosity}, the exchange rate between information and pragmatic value. It offers a mechanism for refining knowledge and action together when grounded in task-relevant symbols and regulated by regret-based potentials. Across decision-oriented monitoring, targeted active search, and composite Bayesian optimization with unknown preferences, PraC reduced downstream decision risk, improved coverage of critical outcome regions, and jointly learned predictive and regret structures without relying on task-specific staging rules.

\bibliography{refs}
\bibliographystyle{tmlr}

\appendix
\section{Proof of Lemma~\ref{lemma:mutual info} (Representation Compression)}\label{appendix: mutual info}

To prove Lemma~\ref{lemma:mutual info}, we first introduce a lemma:

\begin{lemma}\label{lemma:KL-post-pre}
For any (finite or infinite) set $\mathcal{X}$, after a few new measurements $(\mathbf{X},\mathbf{Y}), \mathbf{X} \subseteq \mathcal{X}$ the KL divergence between $p(f_{\mathcal{X}}|\mathcal{D}\cup (\mathbf{X},\mathbf{Y}))$ and $p(f_{\mathcal{X}}|\mathcal{D})$ is 
\begin{equation}
    D_{\text{KL}}[p(f_{\mathcal{X}}|\mathcal{D}\cup (\mathbf{X},\mathbf{Y}))\Vert p(f_{\mathcal{X}}|\mathcal{D})]=\mathbb{E}_{p(f_{\mathbf{X}}|\mathcal{D})}[\frac{L(\mathbf{Y}|f_{\mathbf{X}})}{L(\mathbf{Y})} \log \frac{L(\mathbf{Y}|f_{\mathbf{X}})}{L(\mathbf{Y})}],
\end{equation}
where $L(\mathbf{Y}|f_{\mathbf{X}})$ is the likelihood of observations, $L(\mathbf{Y})=\int L(\mathbf{Y}|f_{\mathbf{X}}) p(f_{\mathbf{X}}|\mathcal{D}) df_{\mathbf{X}}$ the marginal likelihood.
\end{lemma}

\begin{proof}
We first consider the case where $\mathcal{X}$ is \textit{finite} to provide an intuitive insight. 
The KL divergence between $p(f_{\mathcal{X}}|\mathcal{D}\cup (\mathbf{X},\mathbf{Y}))$ and $p(f_{\mathcal{X}}|\mathcal{D})$ can be given as 
\begin{equation}\label{eqn:finite KL}
    \begin{split}
        &D_{\text{KL}}[p(f_{\mathcal{X}}|\mathcal{D}\cup (\mathbf{X},\mathbf{Y}))\Vert p(f_{\mathcal{X}}|\mathcal{D})]\\
        = &D_{\text{KL}}[p(f_{\mathcal{X}\setminus \mathbf{X}}, f_{\mathbf{X}}|\mathcal{D}\cup (\mathbf{X},\mathbf{Y}))\Vert p(f_{\mathcal{X}\setminus \mathbf{X}}, f_{\mathbf{X}}|\mathcal{D})] \\
        =&\int p(f_{\mathcal{X}\setminus \mathbf{X}}, f_{\mathbf{X}}|\mathcal{D}\cup (\mathbf{X},\mathbf{Y})) \log \frac{p(f_{\mathcal{X}\setminus \mathbf{X}}, f_{\mathbf{X}}|\mathcal{D}\cup (\mathbf{X},\mathbf{Y}))}{p(f_{\mathcal{X}\setminus \mathbf{X}}, f_{\mathbf{X}}|\mathcal{D})}df_{\mathcal{X}\setminus \mathbf{X}}df_{\mathbf{X}}\\
        =&\int p(f_{\mathcal{X}\setminus \mathbf{X}}, f_{\mathbf{X}}|\mathcal{D}\cup (\mathbf{X},\mathbf{Y})) \log \frac{p(f_{\mathcal{X}\setminus \mathbf{X}}, f_{\mathbf{X}}|\mathcal{D})L(\mathbf{Y}|f_{\mathbf{X}})}{p(f_{\mathcal{X}\setminus \mathbf{X}}, f_{\mathbf{X}}|\mathcal{D})L(\mathbf{Y})}df_{\mathcal{X}\setminus \mathbf{X}}df_{\mathbf{X}}\\
        =&\int p(f_{\mathcal{X}\setminus \mathbf{X}}, f_{\mathbf{X}}|\mathcal{D}\cup (\mathbf{X},\mathbf{Y})) \log \frac{L(\mathbf{Y}|f_{\mathbf{X}})}{L(\mathbf{Y})}df_{\mathcal{X}\setminus \mathbf{X}}df_{\mathbf{X}} \\
        =&\int p(f_{\mathbf{X}}|\mathcal{D}\cup (\mathbf{X},\mathbf{Y})) \log \frac{L(\mathbf{Y}|f_{\mathbf{X}})}{L(\mathbf{Y})} df_{\mathbf{X}}\\
        =&\int \frac{p(f_{\mathbf{X}}|\mathcal{D})L(\mathbf{Y}|f_{\mathbf{X}})}{L(\mathbf{Y})} \log \frac{L(\mathbf{Y}|f_{\mathbf{X}})}{L(\mathbf{Y})} df_{\mathbf{X}} \\
        =&\mathbb{E}_{p(f_{\mathbf{X}}|\mathcal{D})}[\frac{L(\mathbf{Y}|f_{\mathbf{X}})}{L(\mathbf{Y})} \log \frac{L(\mathbf{Y}|f_{\mathbf{X}})}{L(\mathbf{Y})}],
    \end{split}
\end{equation}
where $L(\mathbf{Y}|f_{\mathbf{X}})$ is the likelihood of observations, $L(\mathbf{Y})=\int L(\mathbf{Y}|f_{\mathbf{X}}) p(f_{\mathbf{X}}|\mathcal{D}) df_{\mathbf{X}}$ the marginal likelihood.

Now we move to the more general cases where $\mathcal{X}$ is \textit{infinite}. In such cases, there is no useful infinite-dimensional Lebesgue measure with respect to an ``infinite-dimensional vector'' $f_{\mathcal{X}}$. 
Thus, we need to resort to a more general definition for KL divergence based on the Radon-Nikodym derivative:
\begin{definition}
If $P$ and $Q$ are probability measures over a set $\mathcal{X}$, and $P$ is absolutely continuous with respect to $Q$, then the KL divergence from $P$ to $Q$ is defined as
$$
D_{\text{KL}}[P\Vert Q]=\int_{\mathcal{X}}\log(\frac{dP}{dQ})dP,
$$
where $\frac{dP}{dQ}$ is the Radon–Nikodym derivative of $P$ with respect to $Q$, and provided the expression on the right-hand side exists.
\end{definition}

According to the measure-theoretic definition of Bayes’ theorem for a dominated model \citep{Schervish1995TheoryStatistics}, the Radon-Nikodym derivative of the posterior $P(\cdot):=p(\cdot|\mathcal{D}\cup (\mathbf{X},\mathbf{Y}))$ with respect to the prior $\hat{P}(\cdot):=p(\cdot|\mathcal{D)}$ is given as 
\begin{equation}
    \frac{d P}{d \hat{P}}(f_{\mathcal{X}}) = \frac{L(\mathbf{Y}|f_{\mathcal{X}})}{L(\mathbf{Y})}.
\end{equation}
Since the dataset $\mathbf{Y}$ is finite, so similar to previous, we restrict the likelihood to only depend on the finite dataset:
\begin{equation}
    \frac{d P}{d \hat{P}}(f_{\mathcal{X}}) = \frac{L(\mathbf{Y}|f_{\mathbf{X}})}{L(\mathbf{Y})}.
\end{equation}
Now the KL divergence between $\hat{P}$ and $P$ is quantified as
\begin{equation}\label{eqn:infinite KL}
    \begin{split}
        &D_{\text{KL}}[P(f_{\mathcal{X}})\Vert \hat{P}(f_{\mathcal{X}})]\\
        =&\int_{f_{\mathcal{X}}} \log 
        (\frac{d P}{d \hat{P}}(f_{\mathcal{X}}))d P(f_{\mathcal{X}}) \\
        =&\int_{f_{\mathcal{X}}} \log 
        (\frac{L(\mathbf{Y}|f_{\mathbf{X}})}{L(\mathbf{Y})})d P(f_{\mathcal{X}}) \\
        =&\int_{f_{\mathcal{X}}} \log 
        (\frac{L(\mathbf{Y}|f_{\mathbf{X}})}{L(\mathbf{Y})}) \frac{L(\mathbf{Y}|f_{\mathbf{X}})}{L(\mathbf{Y})}d \hat{P}(f_{\mathcal{X}}) \\
        =&\int_{f_{\mathbf{X}}} \log 
        (\frac{L(\mathbf{Y}|f_{\mathbf{X}})}{L(\mathbf{Y})}) \frac{L(\mathbf{Y}|f_{\mathbf{X}})}{L(\mathbf{Y})}d \hat{P}(f_{\mathbf{X}}) \\
        =&\mathbb{E}_{\hat{P}(f_{\mathbf{X}})}[\frac{L(\mathbf{Y}|f_{\mathbf{X}})}{L(\mathbf{Y})} \log \frac{L(\mathbf{Y}|f_{\mathbf{X}})}{L(\mathbf{Y})}],
    \end{split}
\end{equation}
which has the exact same form as \eqref{eqn:finite KL}.

Therefore, we can conclude that regardless of the set $\mathcal{X}$ being finite or infinite, the KL divergence between the prior and posterior only depends on the evaluations of the observed data.
That is to say, whilst we are in fact quantifying the KL divergence between the full distributions, we only need to keep track of the distribution over finite function values $f_{\mathbf{X}}$. 
\end{proof}

Now we are ready to prove Lemma~\ref{lemma:mutual info}.

\begin{proof}
According to Lemma~\ref{lemma:KL-post-pre}, 
\begin{equation}\label{eqn:information gain}
    \begin{split}
        I(f_{\mathcal{X}};(\mathbf{X},\mathbf{Y})|\mathcal{D})=
        &\mathbb{E}_{p(\mathbf{Y}| \mathbf{X},\mathcal{D})}[D_{\text{KL}}[p(f_{\mathcal{X}}|\mathcal{D}\cup (\mathbf{X},\mathbf{Y}))\Vert p(f_{\mathcal{X}}|\mathcal{D})]] \\
        = &\mathbb{E}_{p(\mathbf{Y}| \mathbf{X},\mathcal{D})} \mathbb{E}_{p(f_{\mathbf{X}}|\mathcal{D})}[\frac{L(\mathbf{Y}|f_{\mathbf{X}})}{L(\mathbf{Y})} \log \frac{L(\mathbf{Y}|f_{\mathbf{X}})}{L(\mathbf{Y})}] \\
        =&\mathbb{E}_{L(\mathbf{Y})} \mathbb{E}_{p(f_{\mathbf{X}}|\mathcal{D})}[\frac{L(\mathbf{Y}|f_{\mathbf{X}})}{L(\mathbf{Y})} \log \frac{L(\mathbf{Y}|f_{\mathbf{X}})}{L(\mathbf{Y})}] \\
        =&\mathbb{E}_{p(f_{\mathbf{X}}|\mathcal{D})} \mathbb{E}_{L(\mathbf{Y})} [\frac{L(\mathbf{Y}|f_{\mathbf{X}})}{L(\mathbf{Y})} \log \frac{L(\mathbf{Y}|f_{\mathbf{X}})}{L(\mathbf{Y})}] \\
        =&\mathbb{E}_{p(f_{\mathbf{X}}|\mathcal{D})} \int [\frac{L(\mathbf{Y}|f_{\mathbf{X}})}{L(\mathbf{Y})} \log \frac{L(\mathbf{Y}|f_{\mathbf{X}})}{L(\mathbf{Y})}] L(\mathbf{Y}) d\mathbf{Y} \\
        =&\mathbb{E}_{p(f_{\mathbf{X}}|\mathcal{D})} \int L(\mathbf{Y}|f_{\mathbf{X}}) \log \frac{L(\mathbf{Y}|f_{\mathbf{X}})}{L(\mathbf{Y})} d\mathbf{Y} \\
        =&\mathbb{E}_{p(f_{\mathbf{X}}|\mathcal{D})} [D_{\text{KL}}[L(\mathbf{Y}|f_{\mathbf{X}})\Vert L(\mathbf{Y})]] \\
        =&I(f_{\mathbf{X}};\mathbf{Y}|\mathcal{D}).
    \end{split}
\end{equation}
\end{proof}

\section{Re-interpretation of GP-UCB}\label{appx:GP-UCB}

The reinterpretations of most of the acquisition strategies in Table~\ref{tab: unified acquisition strategies} are straightforward according to their definitions. 
However, placing a rather intuitive GP-UCB strategy within this framework seems implicit. 
To reveal their connection, we resort to Lemma~\ref{lemma:mutual info}. Then if we assume constant Gaussian noises $\mathcal{N}\sim(0,\sigma^{2})$ for the observations, we have 
\begin{equation}
    \begin{split}
        I(f_{x};y|\mathcal{D}_{t})&=h(y\mid \mathcal{D}_{t})-H(y|f_{x},\mathcal{D}_{t}) \\
        %=\frac{1}{2}\log (2 \pi e (\sigma_{t}^{2}(x)+\sigma^{2})) - \frac{1}{2}\log (2 \pi e \sigma^{2}) \\
        &=\frac{1}{2}\log(1+\sigma^{-2}\sigma_{t}^{2}(x)),
    \end{split}
\end{equation}
where $\sigma_{t}^{2}(x)$ is the variance evaluated on the GP model $p(f_{x}|\mathcal{D}_{t})$.

When further assuming that the GP kernel $\kappa_{t}(x, x')\le 1, \forall x, x' \in \mathcal{X}$, then $0\le \sigma_{t}^{2}(x)\le \kappa_{t}(x, x') \le 1$, which gives
\begin{equation}
    \log(1+\sigma^{-2}\sigma_{t}^{2}(x)) \ge \log(1+\sigma^{-2})\sigma_{t}^{2}(x).
\end{equation}

If we choose $\beta = \frac{1}{2}\log(1+\sigma^{-2})$, then the epistemic term in PraC, \ie,  $I(f_{x};y|\mathcal{D}_{t})$, provides an upper bound of the square of the exploration term $\beta^{1/2}\sigma_{t}(x)$ in GP-UCB. 

This reveals the close relationship between GP-UCB and PraC, showing that the uncertainty bonus in GP-UCB can be viewed as a first-order surrogate for an information-theoretic epistemic term.

\section{Experimental Details for Decision-Oriented Monitoring}
\label{appendix:decision_source_seeking}

\paragraph{Scope and experimental principle.}
We evaluate decision-oriented monitoring in three domains that differ in downstream tasks but share one sequential sensing pipeline under sparse, noisy plume observation: a finite latent state $\theta$ determines a spatial Poisson field, each query selects one sensor location, and the posterior must support a cost-sensitive downstream action. 

We consider a spatial domain within $[0,100)\times[0,100)$.  Candidate sensor locations form the $20\times20$ grid $\{0,5,\ldots,95\}^2$ and are queried without replacement.  A source $s$ at $\vartheta_s$ produces the hit rate
\begin{equation}
 \lambda_s(x)=
 \frac{R_s\exp[-(\vartheta_s-x)^\top V_s/(2D_s)]
 K_0(\lVert\vartheta_s-x\rVert/\gamma_s)}
 {\log(\gamma_s/a_s)},
 \label{eq:ss-plume-rate}
\end{equation}
where $K_0$ is the modified Bessel function, $a_s=1$, and the implementation uses rate 100 at zero distance.  This intermittent-count model follows the standard source-search construction used by Infotaxis \citep{Vergassola2007Infotaxis}.  With $\Delta t=1$, a measurement is
\begin{equation}
 y_t\mid x_t,\theta\sim
 \operatorname{Poisson}\!\left(\sum_{s\in\theta}
 \lambda_s(x_t)\Delta t\right).
 \label{eq:ss-observation}
\end{equation}
The posterior is represented exactly on the finite hypothesis space and is updated by multiplying the discrete prior by the Poisson likelihood (with only a numerical log-likelihood floor). 

All reported experiments use 20 paired seeds (0--19).  Within a seed, every method receives the same true field, prior, 400-location candidate set, and eight-measurement budget.  Seeds govern Poisson observations and the Random policy.  No method receives an initial measurement or privileged source label.  

\subsection{Domains and datasets}\label{app:ss-tasks}

\paragraph{Source Response Localization.}
After detecting a release, a response team often needs a sufficiently good intervention location rather than a perfectly identified physical parameter. This domain tests whether measurements can reduce response regret while also resolving a deliberately misleading prior.  The true source has $\vartheta=(35,65)$, $R_s=28$, $\gamma=38$, $V=(0.35,-0.15)$, and $D=14$. Location hypotheses use a $20\times20$ grid and strength multipliers $\{0.4,0.7,1.0,1.6\}$, yielding 1,600 hypotheses.  The location prior is
\begin{equation}
 \pi(\vartheta,r)\propto 2\!\times\!10^{-5}+
 \sum_j w_j\exp\!\left[-\frac{\lVert\vartheta-c_j\rVert^2}
 {2\sigma_j^2}\right],
 \label{eq:ss-location-prior}
\end{equation}
uniform across the four strength multipliers, with $(c_j,w_j,\sigma_j)=((75,25),2.5,14)$ and $((35,65),1.0,18)$. The Bayes action is one of 25 dispatch sites $a\in\{0,20,\ldots,80\}^2$. Its hypothesis-conditional loss is $L(a,\theta)=\lVert a-\vartheta\rVert^2/20000$; response success means the chosen site is within 20 distance units of the truth.

\paragraph{Consequence-Weighted Dispatch.}
Equal localization error can have very different operational cost near a hospital, population center, or critical asset.  This task preserves the same sensor, hypothesis, and action spaces while making the loss spatially and source-strength dependent.  The true high-stakes source has $\vartheta=(68,72)$, $R_s=23$, $\gamma=34$, $V=(-0.25,-0.10)$, and $D=14$. The location prior uses the modes in Table~\ref{tab:ss-dispatch-modes} and strength multipliers $\{0.45,0.75,1.0,1.5\}$.  For a hypothesis with strength multiplier $r$,
\begin{align}
 c(\vartheta,r)&=r^{0.5}\left[0.75+
 \sum_j v_j\exp\!\left(-\frac{\lVert\vartheta-d_j\rVert^2}
 {2\tau_j^2}\right)\right],\\
 L(a,\theta)&=c(\vartheta,r)\lVert a-\vartheta\rVert^2/20000.
 \label{eq:ss-consequence-loss}
\end{align}
The truth lies near the largest consequence mode but not the largest prior mode, creating a concrete reason to gather information before committing to a dispatch region.

\begin{table*}[h]
\centering
\caption{Prior and consequence modes for risk-aware dispatch.}
\label{tab:ss-dispatch-modes}
\small
\begin{tabular}{@{}crrrr@{}}
\toprule
Center & Prior weight & Prior scale & Consequence weight & Consequence scale \\
\midrule
$(34,66)$ & 3.2 & 13 & 0.4 & 18 \\
$(68,72)$ & 0.9 & 16 & 4.0 & 16 \\
$(76,24)$ & 1.7 & 14 & 0.8 & 16 \\
\bottomrule
\end{tabular}
\end{table*}

\paragraph{Industrial Leak Repair Prioritization.}
Industrial monitoring commonly exposes more suspected leaks than a repair crew can visit immediately.  The relevant decision is therefore which sources to repair, not merely whether any leak exists.  Six possible sources induce all $2^6-1=63$ nonempty activity hypotheses. The action count for leak repair is $\binom{6}{2}=15$. The prior factorizes into the Bernoulli activity probabilities in Table~\ref{tab:ss-leak-sources}, conditioned on a nonempty configuration.  Only sources 5 and 6 are truly active.  They are weak and assigned low prior activity, but have the largest miss costs; sources 1 and 2 are strong, high-prior decoys.  A repair action selects two sources.  If $p_i$ is the posterior activity marginal and $w_i$ is
the impact weight, the Bayes action selects the two largest $w_i p_i$ and has posterior risk
\begin{equation}
 \operatorname{BR}(D)=\sum_i w_i p_i-
 \sum_{i\in a^*(D)}w_i p_i.
 \label{eq:ss-repair-risk}
\end{equation}
The realized missed-source risk sums $w_i$ over true active sources omitted by the selected repair set.

\begin{table*}[h]
\centering
\caption{Six-source industrial leak-repair instance.  }
\label{tab:ss-leak-sources}
\small
\setlength{\tabcolsep}{3.5pt}
\begin{tabular}{@{}clrrrrrrc@{}}
\toprule
Source & Location & $R_s$ & $\gamma$ & $V_x$ & $V_y$ & $D$ & $(w_i,p_i)$ & True \\
\midrule
1 & $(35,68)$ & 12 & 30 & 0.05 & $-0.10$ & 10 & $(70,0.82)$ & no \\
2 & $(68,70)$ & 12 & 30 & $-0.05$ & $-0.10$ & 10 & $(65,0.80)$ & no \\
3 & $(48,48)$ & 12 & 36 & 0.05 & 0.10 & 12 & $(5,0.40)$ & no \\
4 & $(56,50)$ & 12 & 34 & $-0.05$ & 0.10 & 12 & $(5,0.40)$ & no \\
5 & $(35,72)$ & 4 & 55 & 0.05 & $-0.10$ & 18 & $(180,0.15)$ & yes \\
6 & $(68,74)$ & 4 & 55 & $-0.05$ & $-0.10$ & 18 & $(170,0.15)$ & yes \\
\bottomrule
\end{tabular}
\end{table*}

The overall task configurations are summarized as Table~\ref{tab:ss-frozen-settings}.

\begin{table*}[h]
\centering
\caption{Decision-oriented monitoring task configurations. }
\label{tab:ss-frozen-settings}
\small
\setlength{\tabcolsep}{4pt}
\begin{tabular}{@{}p{5cm}rrrr@{}}
\toprule
Domain & Hypotheses & Actions & Candidates & Budget  \\
\midrule
Response localization & 1,600 & 25 & 400 & 8 \\
Consequence-Weighted dispatch & 1,600 & 25 & 400 & 8 \\
Leak Repair Prioritization & 63 & 15 & 400 & 8 \\
\bottomrule
\end{tabular}
\end{table*}

\subsection{Acquisition methods}\label{app:ss-baselines}

The shared pragmatic quantity is the expected reduction in downstream
Bayes risk,
\begin{align}
 \Delta \operatorname{BR}_t(x)&=\operatorname{BR}(\mathcal D_t)-\mathbb E_{q(y\mid x,\mathcal D_t)}
   [\operatorname{BR}(\mathcal D_t\cup\{(x,y)\})].
 \label{eq:ss-pragmatic}
\end{align}

For posterior $q_t(\theta)$, candidate $x$, and possible count $y$, let $q_t(\theta\mid y,x)$ be the exact hypothetical posterior.  The epistemic term is mutual information about the \emph{entire} finite latent state,
\begin{equation}
  I_t(x)=H[q_t(\theta)]-
 \mathbb E_{q(y\mid x,\mathcal D_t)}H[q_t(\theta\mid y,x)].
 \label{eq:ss-information}
\end{equation}
For localization and dispatch, the dataset-knowledge uncertainty $U_t$ used by the dynamic curiosity scheduler is normalized entropy of the response-site symbol induced by $q_t(\theta)$; for leak repair it is the mean normalized Bernoulli entropy of the six activity marginals.  

\paragraph{Acquisition methods.} The six compared policies are: (i) \textsc{Random}, uniform sampling without replacement; (ii) \textsc{Greedy}, maximizing $\Delta \operatorname{BR}_t$; (iii) \textsc{EIG}, maximizing $ I_t$; (iv) \textsc{IDS}, minimizing $ \operatorname{BR}^2_t(x)/(I_t(x)+\epsilon_I)$ over informative candidates, with minimum-$\operatorname{BR}_t$ fallback when none is informative \citep{Russo2018IDS}; (v) \textsc{PraC fixed beta}, with $\beta=5$; and (vi) \textsc{PraC dynamic beta}, with a dynamic curiosity scheduler where $q=0.75$, $U_\star=0$, $\epsilon_I=\epsilon_U=10^{-8}$, $\beta_{\min}=10^{-3}$, and $\beta_{\max}=10$, $k_\beta=1$ in the two response tasks and $k_\beta=4$ in prioritization.  

\paragraph{Stable ties and pairing.} For acquisition calculations, the predictive count support is summed exactly from zero through $\lceil\lambda_{\max}+6\sqrt{\lambda_{\max}}+1\rceil$, with $\lambda_{\max}\geq20$ used to retain a conservative tail.  No posterior or acquisition Monte Carlo approximation is used.  Non-random methods take the first maximizer on the fixed grid ordering; Random samples uniformly from the remaining locations.

\subsection{Evaluation metrics}\label{app:ss-metrics}

The primary metric is \textbf{posterior Bayes risk}, because it directly
measures how much expected loss remains under the current posterior and the Bayes-optimal action.  We report both its final value and normalized AUC across the eight post-query observations,
\begin{equation}
 \operatorname{AUC}_{\mathrm{BR}}=
 \frac{1}{7}\sum_{t=1}^7
 \frac{\operatorname{BR}_t+\operatorname{BR}_{t+1}}{2},
 \label{eq:ss-risk-auc}
\end{equation}
where lower is better.  The common zero-query state is not inserted into AUC. 

We also report \textbf{realized downstream loss}.  For the two response tasks this is the true loss of the current Bayes dispatch action, plus response success and MAP location error.  For leak repair we report missed-source risk, unweighted top-2 recall, and impact-weighted top-2 recall.

\subsection{Detailed results}\label{app:ss-detail-results}

\begin{figure*}[h]
\centering
\includegraphics[width=\textwidth]{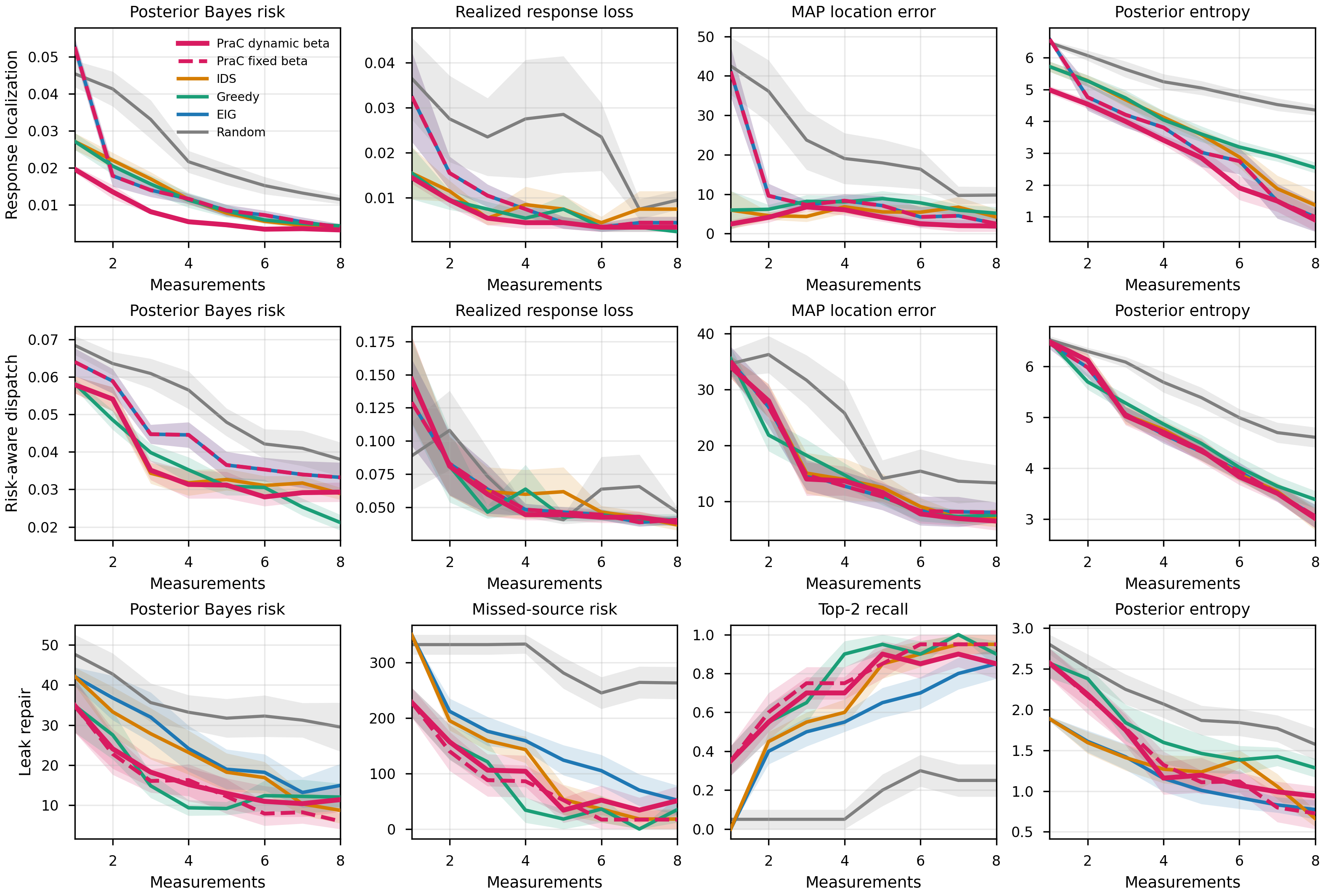}
\caption{Complete metrics exposing posterior Bayes
risk, realized loss, a task-specific diagnostic, and full posterior entropy. Lines and shaded regions are mean $\pm$ SEM over the same 20 paired seeds.}
\label{fig:ss-detailed-trajectories}
\end{figure*}

\paragraph{Response localization: curiosity corrects a misleading belief.}
The initial prior assigns greater probability to an incorrect source region, so a purely pragmatic policy can benefit from resolving uncertainty before committing to a response. Dynamic PraC produces the fastest reduction in posterior Bayes risk, with mean AUC $0.0072$, compared with $0.0118$ for Greedy, $0.0120$ for IDS, $0.0133$ for fixed PraC/EIG, and $0.0245$ for Random (Table~\ref{tab:ss-response-results}). It also attains the lowest final Bayes risk and the smallest mean MAP localization error. The result illustrates the intended interaction between the two terms of PraC: epistemic pressure resolves ambiguity in the latent source state, while expected risk reduction concentrates sensing on distinctions that affect the response action.

\begin{table*}[h]
\centering
\caption{Source Response Localization results (mean $\pm$ SEM, 20 seeds). Bayes-risk AUC and final Bayes risk are lower-is-better; other columns are values after eight measurements.}
\label{tab:ss-response-results}
\small
\begin{tabular}{@{}lrrrrr@{}}
\toprule
Method & Bayes-risk AUC & Final Bayes risk & Response loss & Response success & MAP error \\
\midrule
PraC dynamic beta & $0.0072 \pm 0.0005$ & $0.0033 \pm 0.0005$ & $0.0035 \pm 0.0010$ & $1.000 \pm 0.000$ & $1.91 \pm 1.45$ \\
PraC fixed beta & $0.0133 \pm 0.0012$ & $0.0040 \pm 0.0009$ & $0.0045 \pm 0.0013$ & $1.000 \pm 0.000$ & $2.50 \pm 1.34$ \\
IDS & $0.0120 \pm 0.0011$ & $0.0041 \pm 0.0008$ & $0.0075 \pm 0.0040$ & $0.900 \pm 0.100$ & $4.20 \pm 1.97$ \\
Greedy & $0.0118 \pm 0.0014$ & $0.0043 \pm 0.0005$ & $0.0025 \pm 0.0000$ & $1.000 \pm 0.000$ & $5.15 \pm 1.61$ \\
EIG & $0.0133 \pm 0.0012$ & $0.0040 \pm 0.0009$ & $0.0045 \pm 0.0013$ & $1.000 \pm 0.000$ & $2.50 \pm 1.34$ \\
Random & $0.0245 \pm 0.0028$ & $0.0115 \pm 0.0013$ & $0.0095 \pm 0.0021$ & $0.900 \pm 0.100$ & $9.77 \pm 2.10$ \\
\bottomrule
\end{tabular}
\end{table*}

\paragraph{Consequence-weighted dispatch: pragmatic structure redirects information acquisition.}
Introducing spatially varying consequences changes which uncertainty matters for action. Greedy, IDS, and dynamic PraC form the strongest group in Bayes-risk reduction, while dynamic PraC improves substantially over fixed PraC, EIG, and Random (Table~\ref{tab:ss-dispatch-results}). The important change from localization is that the loss landscape itself now provides informative structure: measurements near high-consequence hypotheses acquire greater pragmatic value even before uncertainty is fully resolved. PraC therefore moves naturally toward a more decision-focused sensing policy as the downstream loss becomes more informative.

\begin{table*}[h]
\centering
\caption{Consequence-Weighted Dispatch results (mean $\pm$ SEM, 20 seeds). Bayes-risk AUC and final Bayes risk are lower-is-better; other columns are values after eight measurements.}
\label{tab:ss-dispatch-results}
\small
\begin{tabular}{@{}lrrrrr@{}}
\toprule
Method & Bayes-risk AUC & Final Bayes risk & Response loss & Response success & MAP error \\
\midrule
PraC dynamic beta & $0.0360 \pm 0.0017$ & $0.0293 \pm 0.0026$ & $0.0386 \pm 0.0031$ & $1.000 \pm 0.000$ & $6.50 \pm 1.66$ \\
PraC fixed beta & $0.0432 \pm 0.0020$ & $0.0332 \pm 0.0040$ & $0.0405 \pm 0.0032$ & $1.000 \pm 0.000$ & $8.04 \pm 1.79$ \\
IDS & $0.0370 \pm 0.0019$ & $0.0289 \pm 0.0016$ & $0.0366 \pm 0.0041$ & $1.000 \pm 0.000$ & $7.04 \pm 1.61$ \\
Greedy & $0.0357 \pm 0.0015$ & $0.0212 \pm 0.0021$ & $0.0405 \pm 0.0043$ & $1.000 \pm 0.000$ & $7.57 \pm 0.93$ \\
EIG & $0.0432 \pm 0.0020$ & $0.0332 \pm 0.0040$ & $0.0405 \pm 0.0032$ & $1.000 \pm 0.000$ & $8.04 \pm 1.79$ \\
Random & $0.0522 \pm 0.0025$ & $0.0380 \pm 0.0046$ & $0.0463 \pm 0.0039$ & $1.000 \pm 0.000$ & $13.32 \pm 3.17$ \\
\bottomrule
\end{tabular}
\end{table*}

\paragraph{Leak repair: high-impact hypotheses require sustained curiosity.}
The active sources are deliberately weak, assigned low prior probability, and carry the largest miss costs. This creates a setting in which exploiting the current activity marginals can overlook precisely the hypotheses that matter most. Fixed PraC with $\beta=5$ achieves the lowest mean Bayes-risk AUC, final Bayes risk, and missed-source risk, together with the highest impact-weighted recall (Table~\ref{tab:ss-repair-results}). Dynamic PraC still reduces Bayes-risk AUC substantially relative to EIG and Random, but the strongest performance comes from maintaining epistemic pressure throughout the short eight-query horizon. This result complements the first two tasks by showing that the appropriate curiosity level is determined by the structure of the unresolved decision-relevant uncertainty.

\begin{table*}[h]
\centering
\caption{Industrial Leak Repair Prioritization results (mean $\pm$ SEM, 20 seeds).  Recall columns are higher-is-better; risk columns are lower-is-better.}
\label{tab:ss-repair-results}
\small
\begin{tabular}{@{}lrrrrr@{}}
\toprule
Method & Bayes-risk AUC & Final Bayes risk & Missed risk & Top-2 recall & Weighted recall \\
\midrule
PraC dynamic beta & $16.42 \pm 2.76$ & $11.40 \pm 3.07$ & $51.0 \pm 26.0$ & $0.850 \pm 0.076$ & $0.854 \pm 0.074$ \\
PraC fixed beta & $14.85 \pm 2.70$ & $6.01 \pm 1.94$ & $17.0 \pm 17.0$ & $0.950 \pm 0.050$ & $0.951 \pm 0.049$ \\
IDS & $22.17 \pm 3.30$ & $8.73 \pm 3.23$ & $18.0 \pm 18.0$ & $0.950 \pm 0.050$ & $0.949 \pm 0.051$ \\
Greedy & $15.57 \pm 2.46$ & $12.05 \pm 3.41$ & $35.0 \pm 23.3$ & $0.900 \pm 0.067$ & $0.900 \pm 0.067$ \\
EIG & $24.54 \pm 3.08$ & $14.95 \pm 5.41$ & $52.0 \pm 26.5$ & $0.850 \pm 0.076$ & $0.851 \pm 0.076$ \\
Random & $35.03 \pm 4.55$ & $29.50 \pm 6.10$ & $263.0 \pm 29.0$ & $0.250 \pm 0.083$ & $0.249 \pm 0.083$ \\
\bottomrule
\end{tabular}
\end{table*}

\paragraph{Paired comparisons.}
Pairwise comparisons in Table~\ref{tab:ss-paired-tests} use two-sided paired Wilcoxon signed-rank tests on per-seed Bayes-risk AUC.  We define $\Delta=\operatorname{AUC}_{\mathrm{baseline}}-\operatorname{AUC}_{\mathrm{dynamic}}$, so positive values favor dynamic PraC.  Holm adjustment is applied over the five dynamic-PraC contrasts within each application.  A paired $t$ interval is included as a descriptive 95\% interval for the mean difference; inferential claims use the Wilcoxon test. 

\begin{table*}[h]
\centering
\caption{Paired Bayes-risk-AUC comparisons.  Positive $\Delta$ favors dynamic PraC.  Wins/ties counts seeds with positive/zero paired difference; Holm adjustment is within application over the five displayed contrasts.}
\label{tab:ss-paired-tests}
\small
\setlength{\tabcolsep}{4pt}
\begin{tabular}{@{}llrrrrr@{}}
\toprule
Application & Baseline & $\Delta$ AUC & 95\% paired $t$-interval & Wins/ties & $p$ & $p_{\mathrm{Holm}}$ \\
\midrule
Response localization & PraC fixed beta & $0.0061$ & $[0.0027, 0.0095]$ & 18/0 & $0.010$ & $0.029$ \\
Response localization & IDS & $0.0048$ & $[0.0020, 0.0077]$ & 18/0 & $0.010$ & $0.029$ \\
Response localization & Greedy & $0.0046$ & $[0.0012, 0.0080]$ & 19/0 & $0.004$ & $0.016$ \\
Response localization & EIG & $0.0061$ & $[0.0027, 0.0095]$ & 18/0 & $0.010$ & $0.029$ \\
Response localization & Random & $0.0173$ & $[0.0111, 0.0234]$ & 20/0 & $1.91e-6$ & $9.55e-6$ \\
Risk-aware dispatch & PraC fixed beta & $0.0072$ & $[0.0016, 0.0127]$ & 19/0 & $0.010$ & $0.039$ \\
Risk-aware dispatch & IDS & $0.0009$ & $[-0.0017, 0.0036]$ & 15/1 & $0.374$ & $0.749$ \\
Risk-aware dispatch & Greedy & $-0.0004$ & $[-0.0035, 0.0028]$ & 15/0 & $0.922$ & $0.922$ \\
Risk-aware dispatch & EIG & $0.0072$ & $[0.0016, 0.0127]$ & 19/0 & $0.010$ & $0.039$ \\
Risk-aware dispatch & Random & $0.0161$ & $[0.0111, 0.0211]$ & 20/0 & $1.91e-6$ & $9.55e-6$ \\
Leak repair & PraC fixed beta & $-1.5729$ & $[-6.0652, 2.9194]$ & 15/1 & $0.767$ & $1.000$ \\
Leak repair & IDS & $5.7454$ & $[-2.5095, 14.0002]$ & 17/0 & $0.232$ & $0.697$ \\
Leak repair & Greedy & $-0.8489$ & $[-6.7038, 5.0060]$ & 14/1 & $0.953$ & $1.000$ \\
Leak repair & EIG & $8.1153$ & $[0.4556, 15.7749]$ & 19/0 & $0.064$ & $0.258$ \\
Leak repair & Random & $18.6111$ & $[6.8491, 30.3731]$ & 19/0 & $0.010$ & $0.049$ \\
\bottomrule
\end{tabular}
\end{table*}

\subsection{Curiosity-sensitivity analysis}
\label{app:ss-beta-sensitivity}

We separately test whether dynamic PraC reduces the need to identify a single problem-specific fixed curiosity weight. Every run uses the same task definitions and eight-measurement budget, with 20 paired seeds.  

Response localization and risk-aware dispatch sweep $q\in\{0.5,0.75,0.9\}$ and $k_\beta\in\{0.5,1,2\}$.  Because useful epistemic pressure is larger in leak repair, we sweep $q\in\{0.75,0.9,0.95\}$ and $k_\beta\in\{1,2,4\}$.  All remaining scheduler constants match the default setting in the main experiments.  The fixed-PraC comparison sweeps $\beta\in\{0,0.1,1,5,10\}$. 

For task $j$, let $A_j(q,k)$ denote mean normalized Bayes-risk AUC for one dynamic setting cell and $A_j^*=\min_{q,k}A_j(q,k)$.  Cell regret is
\begin{equation}
 R_j(q,k)=100\left[\frac{A_j(q,k)}{A_j^*}-1\right]\%,
 \label{eq:ss-dynamic-regret}
\end{equation}
so zero is task-best and smaller is more robust.  The joint-grid sensitivity is $S_{\mathrm{grid},j}=\max_{q,k}R_j(q,k)$.  To avoid attributing poor choices of the other parameter to a one-dimensional profile, we first profile it out:
\begin{align}
 A_j^q(q)&=\min_k A_j(q,k), &
 S_{q,j}&=\max_q 100\left[A_j^q(q)/A_j^*-1\right],\\
 A_j^k(k)&=\min_q A_j(q,k), &
 S_{k_\beta,j}&=\max_k 100\left[A_j^k(k)/A_j^*-1\right].
 \label{eq:ss-profile-sensitivity}
\end{align}
For fixed PraC, $S_{\beta,j}$ is the maximum analogous regret over the five fixed values.  A setting is called robust when its regret is at most 5\%; the robust fraction is reported relative to nine tested dynamic cells or five fixed values. 

\begin{figure*}[h]
\centering
\includegraphics[width=\textwidth]{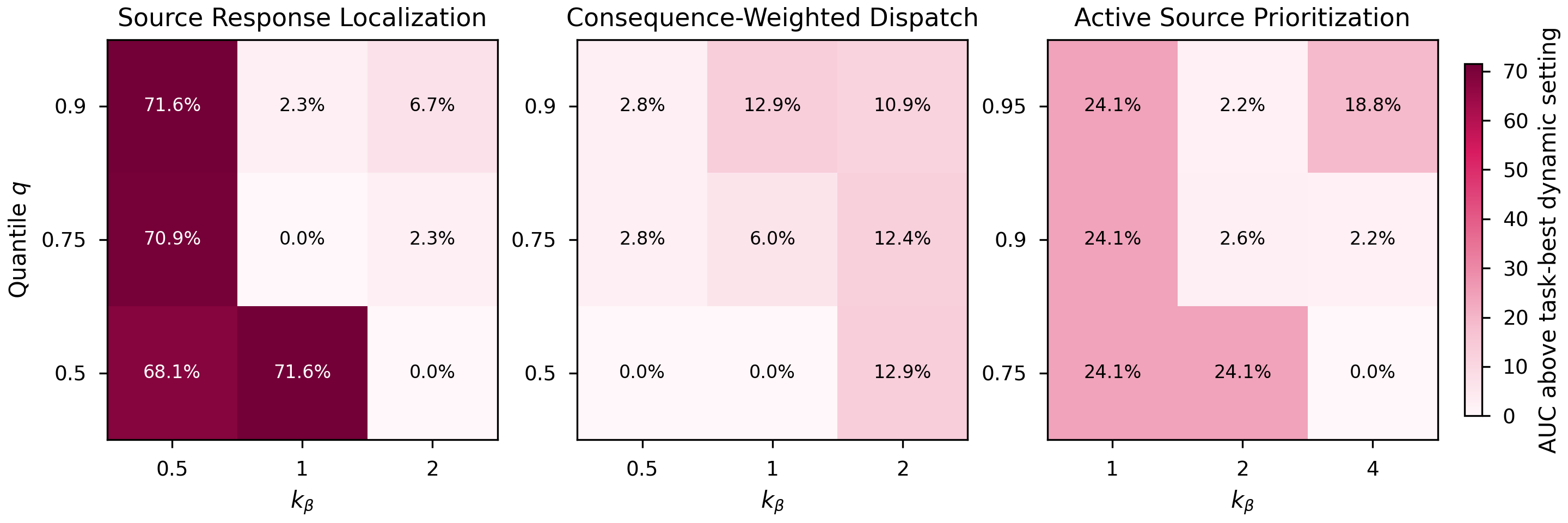}
\caption{Dynamic-PraC sensitivity over the task-specific $3\times3$ grids. Each cell is percentage Bayes-risk-AUC regret relative to the best dynamic setting for that task; lighter cells are less sensitive.  The response and dispatch grids use $q=\{0.5,0.75,0.9\}$, $k_\beta=\{0.5,1,2\}$; leak repair uses $q=\{0.75,0.9,0.95\}$, $k_\beta=\{1,2,4\}$.}
\label{fig:ss-beta-sensitivity}
\end{figure*}

\begin{table*}[h]
\centering
\caption{Finite-grid hyperparameter sensitivity.  $S_q$ and $S_{k_\beta}$ profile over the other dynamic parameter.  Lower maximum regret and a larger within-5\% fraction indicate less tuning sensitivity.  The final row reports the mean task sensitivity and pooled robust-setting counts.}
\label{tab:ss-beta-sensitivity}
\small
\setlength{\tabcolsep}{4pt}
\begin{tabular}{@{}lrrrrrrr@{}}
\toprule
& \multicolumn{4}{c}{Dynamic PraC} & \multicolumn{3}{c}{Fixed PraC} \\
Task & $S_q$ & $S_{k_\beta}$ & $S_{\mathrm{grid}}$ & Within 5\% & $\beta^*$ & $S_\beta$ & Within 5\% \\
\midrule
Response localization & $2.3$ & $68.1$ & $71.6$ & 4/9 (44\%) & $0.1$ & $81.3$ & 1/5 (20\%) \\
Risk-aware dispatch & $2.8$ & $10.9$ & $12.9$ & 4/9 (44\%) & $0$ & $21.8$ & 1/5 (20\%) \\
Leak repair & $2.2$ & $24.1$ & $24.1$ & 4/9 (44\%) & $5$ & $37.6$ & 1/5 (20\%) \\
Across tasks & $2.4\ (mean)$ & $34.4\ (mean)$ & $36.2\ (mean)$ & 12/27 (44\%) & \textit{none} & $46.9\ (mean)$ & 3/15 (20\%) \\
\bottomrule
\end{tabular}
\end{table*}

\paragraph{Sensitivity analysis.}
The three monitoring tasks favor markedly different fixed curiosity weights: $\beta=0.1$ for localization, $\beta=0$ for consequence-weighted dispatch, and $\beta=5$ for leak repair. This variation reflects the changing role of unresolved uncertainty in the downstream decision. The dynamic scheduler partially absorbs this task dependence by estimating the local information–regret scale online. Across the tested grids, 12 of 27 dynamic settings lie within 5\% of their task-specific optimum, compared with 3 of 15 fixed settings, and the dynamic grid has lower worst-case relative degradation in every task (Table~\ref{tab:ss-beta-sensitivity}). In particular, the quantile parameter is stable after profiling over the feedback gain, with only $2.2$–$2.8\%$ worst-case regret across tasks. These results support the feedback–feedforward interpretation of Section~\ref{subsec:curiosity}: curiosity can be regulated from the current epistemic and pragmatic state rather than fixed once for all decision problems.

\section{Experimental Details for Targeted Active Search}
\label{app:constrained-active-search}

\paragraph{Scope and experimental principle.}
We evaluate targeted active search in three domains that differ in scientific object, representation, and oracle, but share the same decision structure: a finite pool contains rare targets partitioned into families, the initial data show only a restricted subset of those families, and each sequential query is costly.  The goal is therefore not only to accumulate positives but also to cover target families that are absent from the initialization.  

All reported experiments use 20 paired seeds (0--19).  Within a seed, every method receives the same initialization, frozen candidate pool, oracle, representation, and query budget.  Dataset labels used to define targets or evaluation families do not enter the representation or covariance construction.

\subsection{Domains and datasets}\label{app:tas-tasks}

\paragraph{Drug discovery: BACE scaffold hopping.}
Medicinal-chemistry campaigns frequently begin in one productive chemical series, but remaining in that series can miss chemically distinct active scaffolds.  This makes scaffold hopping a natural test of whether epistemic exploration can trade a small amount of local potency exploitation for broader coverage of active chemotypes.  We use the BACE regression table distributed with MoleculeNet \citep{Wu2018MoleculeNet}.  A molecule is a target when $\mathrm{pIC}_{50}\geq 8$.  Molecules are represented by radius-2, 2,048-bit Morgan fingerprints \citep{Rogers2010ECFP}, and generic Bemis--Murcko frameworks \citep{Bemis1996Frameworks} define families.  We retain, without reading activity, families with at least five members.  The resulting frozen pool has 1,040 molecules, 132 targets, 72 retained families, and 32 target-containing families.  The largest eligible family is the anchor (63 molecules and six targets); initialization samples one target and four non-targets from this family.  The other 31 target-containing families are hidden at initialization.  Each run makes 50 sequential assays.

\paragraph{Materials discovery: high-$T_c$ superconductor prototype hopping.}
Experimental synthesis and critical-temperature measurement are expensive, while rare high-$T_c$ candidates occur in multiple structural prototypes. Prototype hopping therefore retains the same exploration--exploitation structure as scaffold hopping in a materially different domain.  We use the open 3DSC$_{\mathrm{MP}}$ table, which augments compositions and critical temperatures with approximate Materials Project structures \citep{Sommer2023ThreeDSC}.  Before reading $T_c$, we retain formula-derived cuprates in prototype families (\texttt{material\_id\_2}) with at least five full-dataset members.  A target satisfies $T_c\geq95\,$K.  The frozen pool contains 672 materials, 30 targets (4.46\%), 39 prototypes, and six target-containing prototypes.  The anchor is \texttt{mp-1218930} (65 candidates, one target); initialization contains that target and four non-targets from the same prototype.  Five target prototypes are consequently hidden.  Each run makes 50 sequential measurements.

\paragraph{Industrial inspection: partial-supervision defect-mechanism coverage.}
In inspection, not every anomaly warrants the same engineering response, we often seek to search defect images that reveal actionable mechanisms rather than non-actionable anomalies.  We use the complete 150-image \texttt{cable} test split of MVTec AD \citep{Bergmann2019MVTecAD}.  The 44 actionable ``signal-path integrity'' images comprise \texttt{cut\_inner\_insulation}, \texttt{cut\_outer\_insulation}, \texttt{missing\_wire}, and \texttt{poke\_insulation}.  The known target types are \texttt{cut\_inner\_insulation} and \texttt{missing\_wire}; the other two are hidden.  Calibration contains one positive from each known type, two test good images, and two defects from distinct non-actionable mechanisms.  These six calibration images are excluded from all evaluation numerators and denominators.  Twenty disjoint \texttt{train/good} images define normal appearance but are neither queries nor discoveries.  Each run makes 30 sequential inspections.

The overall task configurations are summarized as Table~\ref{tab:cas-frozen-settings}.

\begin{table*}[h]
\centering
\caption{Constrained active search task configurations.  ``Hidden families'' excludes
families represented by positive calibration examples.}
\label{tab:cas-frozen-settings}
\small
\setlength{\tabcolsep}{4pt}
\begin{tabular}{@{}p{2cm}p{2cm}p{1cm}p{1cm}p{1.2cm}p{1.2cm}p{4.5cm}p{1cm}@{}}
\toprule
Domain & Target oracle & Pool & Targets & Target families & Hidden families & Initialization & Budget \\
\midrule
BACE scaffold hopping & $\mathrm{pIC}_{50}\geq 8$ & 1,040 & 132 & 32 & 31 & 1 target + 4 non-targets in anchor scaffold & 50 \\
3DSC-95K prototype hopping & $T_c\geq95\,$K & 672 & 30 & 6 & 5 & 1 target + 4 non-targets in anchor prototype & 50 \\
MVTec cable mechanism coverage & Membership in four actionable types & 150 & 44 & 4 & 2 & 2 known-type targets + 2 good + 2 hard negatives & 30 \\
\bottomrule
\end{tabular}
\end{table*}

\subsection{Acquisition methods}\label{app:tas-baselines}

At decision $t$, an exact Gaussian-process surrogate is refit to the observed scalar oracle outcomes.  BACE and 3DSC use the continuous $\mathrm{pIC}_{50}$ and $T_c$ values; MVTec uses the binary actionable indicator with threshold 0.5.  Outcomes are standardized using the observations available at that decision.  

The epistemic term quantifies uncertainty reduction over the \emph{entire} finite parameter space.  If $C_t$ is the posterior covariance of the latent outcome at all pool elements, then querying $x$ has A-optimal information value
\begin{equation}
 I_t(x)=\operatorname{tr}(C_t-C_{t+1\mid x})
 =\frac{\sum_j C_t[j,x]^2}{C_t[x,x]+\sigma_{\mathrm{obs}}^2}.
 \label{eq:cas-whole-pool-information}
\end{equation}
The same covariance defines the normalized dataset-knowledge uncertainty
$U_t=n^{-1}\sum_j C_t[j,j]/K[j,j]$ used by the dynamic curiosity scheduler.

\paragraph{Acquisition methods.}
The five compared policies are: (i) \textsc{Random}, uniform sampling without replacement; (ii) \textsc{Greedy}, which maximizes $-h_t(x)$; (iii) \textsc{EIG}, which maximizes $I_t(x)$; (iv) \textsc{IDS}, the deterministic information-ratio rule maximizing $-(h_t(x))^2/(I_t(x)+\epsilon_I)$; and (v) \textsc{PraC}, with default dynamic curiosity scheduler where $q=0.75$, $k_\beta=1$, $U_\star=0$, $\epsilon_I=\epsilon_U=10^{-8}$, $\beta_{\min}=10^{-3}$, and $\beta_{\max}=10$. 

\paragraph{Stable ties and pairing.}
Candidates within relative tolerance $10^{-10}$ and absolute tolerance $10^{-12}$ of the maximum score form a tie set.  A BLAKE2 hash of the seed, decision index, and stable candidate key selects within the set.  This removes row-order leakage while preserving reproducibility and cross-method pairing.

\subsection{Evaluation metrics}\label{app:tas-metrics}

For BACE and 3DSC, let $\mathcal{H}$ be target-containing families other than the initialized anchor, $N_f$ the number of targets in family $f$, and $D_f(t)$ the targets from $f$ observed by decision $t$.  The primary metric is \textbf{hidden-family macro recall},
\begin{equation}
 R_{\mathrm{macro}}^{\mathrm{hidden}}(t)
 =\frac{1}{|\mathcal{H}|}\sum_{f\in\mathcal{H}}
 \frac{|D_f(t)|}{N_f}.
 \label{eq:cas-hidden-macro-recall}
\end{equation}
Equal family weighting prevents a large active series from dominating the endpoint.  We also report \textbf{hidden target families discovered} $\sum_{f\in\mathcal{H}}\mathbf{1}[|D_f(t)|>0]$, hidden target hits, and anchor-family recall.

For MVTec, the primary metric is \textbf{calibration-excluded macro recall} over the four actionable types.  The denominator for a known type excludes its one calibration positive; hidden-type denominators are unchanged.  We separately report \textbf{hidden types covered}, known- and hidden-type macro recall, actionable hits,  query precision, and the fraction of queries spent on non-actionable defects.  All MVTec curves start from the shared evaluation state at query zero; no calibration example is credited.

Normalized AUC is trapezoidal area divided by the plotted horizon.  The frozen BACE and 3DSC artifacts record queries 1--50 and use 49 intervals; the MVTec artifact prepends the common zero state and uses 30 intervals over queries 0--30.   

\subsection{Detailed results}\label{app:tas-detail-results}

\begin{figure*}[h]
\centering
\includegraphics[width=\textwidth]{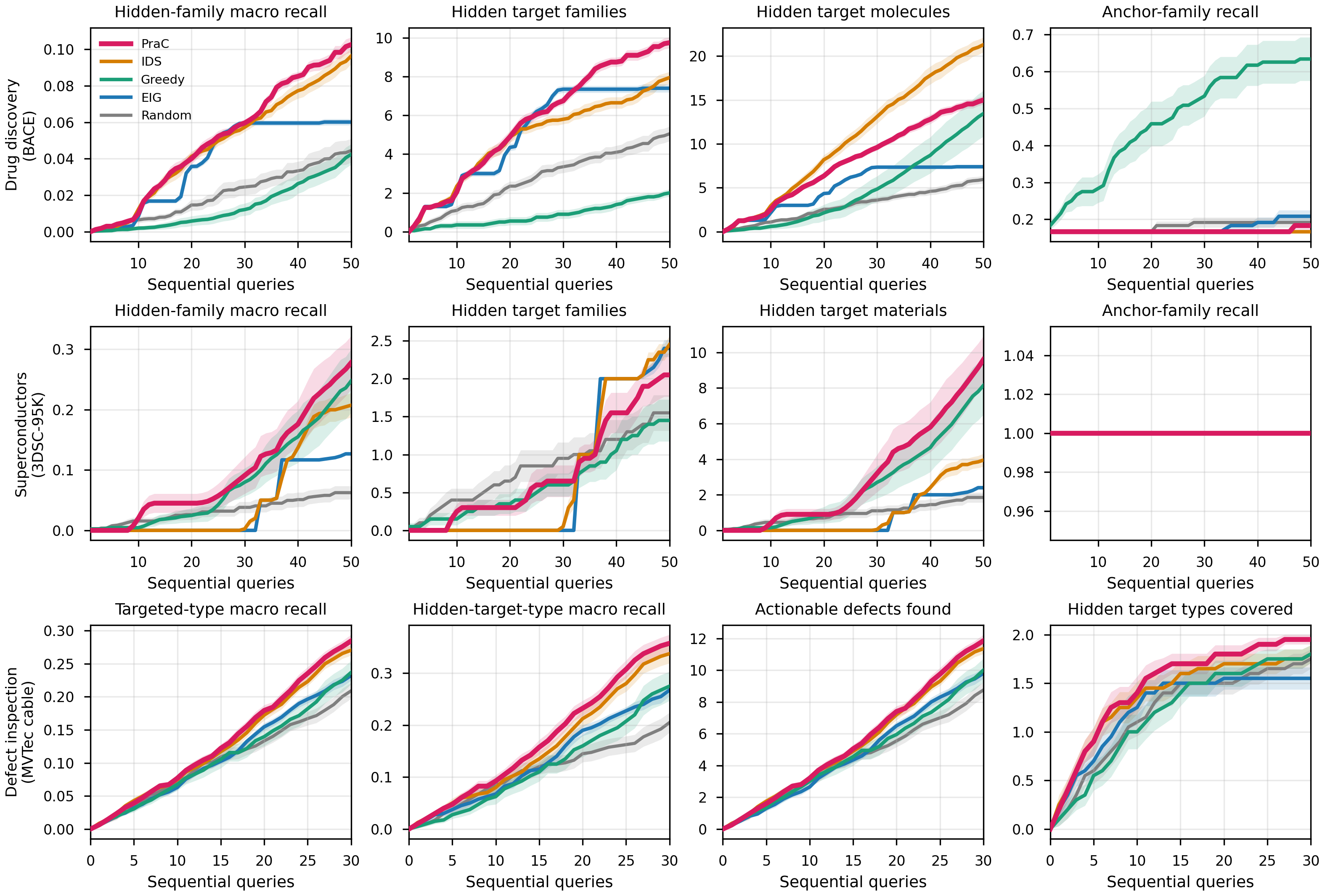}
\caption{Complete metrics exposing exploration, target yield, and
anchor/known-family exploitation.  Lines and shaded regions are mean $\pm$ SEM over the same 20 paired seeds.}
\label{fig:cas-detailed-trajectories}
\end{figure*}

\paragraph{BACE: moving beyond a productive scaffold.}
The initialization contains a positive example only from the anchor scaffold, creating a natural tension between continuing to exploit a known productive family and exploring chemically distinct regions. Greedy strongly favors the anchor, reaching \(0.633\) anchor recall but only \(2.00\) hidden families. PraC shifts queries outward: it discovers \(9.75\) hidden scaffolds and attains the highest mean hidden-family macro-recall AUC, while still finding \(15.00\) hidden active molecules (Table~\ref{tab:cas-bace-results}). IDS achieves a larger raw target count, \(21.30\), but across fewer hidden families. The distinction is central to the evolving coverage potential: PraC values an active discovery not only by whether it is positive, but by how much new outcome structure it adds beyond what has already been observed.

\begin{table*}[h]
\centering
\caption{BACE 20-seed results (mean $\pm$ SEM).  AUC is the primary
hidden-family macro-recall AUC; all other columns are values after 50 queries.}
\label{tab:cas-bace-results}
\small
\begin{tabular}{@{}lrrrrr@{}}
\toprule
Method & Primary AUC & Hidden macro recall & Hidden families & Hidden targets & Anchor recall \\
\midrule
PraC & $0.051 \pm 0.002$ & $0.103 \pm 0.004$ & $9.75 \pm 0.28$ & $15.00 \pm 0.49$ & $0.183 \pm 0.011$ \\
IDS & $0.048 \pm 0.002$ & $0.096 \pm 0.004$ & $7.95 \pm 0.30$ & $21.30 \pm 0.84$ & $0.167 \pm 0.000$ \\
Greedy & $0.013 \pm 0.003$ & $0.042 \pm 0.006$ & $2.00 \pm 0.19$ & $13.45 \pm 2.62$ & $0.633 \pm 0.059$ \\
EIG & $0.039 \pm 0.001$ & $0.060 \pm 0.002$ & $7.40 \pm 0.21$ & $7.40 \pm 0.21$ & $0.208 \pm 0.017$ \\
Random & $0.020 \pm 0.003$ & $0.044 \pm 0.007$ & $5.05 \pm 0.38$ & $5.95 \pm 0.43$ & $0.192 \pm 0.014$ \\
\bottomrule
\end{tabular}
\end{table*}

\paragraph{3DSC-95K: balancing depth and breadth for rare targets.}
Only \(30\) of \(672\) materials are targets and five target-bearing prototypes are hidden at initialization. PraC obtains the highest mean hidden-prototype macro-recall AUC (\(0.096\)) and the largest number of hidden target materials (\(9.60\)). IDS and EIG touch slightly more distinct hidden prototypes, while PraC recovers a substantially larger fraction of the targets within the prototypes it reaches (Table~\ref{tab:cas-three-dsc-results}). Consequently, the family-balanced recall favors PraC: the acquisition combines exploration across uncertain prototype structure with continued value for under-covered target regions.

\begin{table*}[h]
\centering
\caption{3DSC-95K 20-seed results (mean $\pm$ SEM).  AUC is the primary hidden-prototype macro-recall AUC; all other columns are values after 50
queries.}
\label{tab:cas-three-dsc-results}
\small
\begin{tabular}{@{}lrrrrr@{}}
\toprule
Method & Primary AUC & Hidden macro recall & Hidden families & Hidden targets & Anchor recall \\
\midrule
PraC & $0.096 \pm 0.023$ & $0.277 \pm 0.044$ & $2.05 \pm 0.29$ & $9.60 \pm 1.33$ & $1.000 \pm 0.000$ \\
IDS & $0.051 \pm 0.004$ & $0.207 \pm 0.020$ & $2.45 \pm 0.11$ & $3.95 \pm 0.37$ & $1.000 \pm 0.000$ \\
Greedy & $0.079 \pm 0.020$ & $0.247 \pm 0.053$ & $1.45 \pm 0.28$ & $8.15 \pm 1.65$ & $1.000 \pm 0.000$ \\
EIG & $0.037 \pm 0.000$ & $0.127 \pm 0.003$ & $2.40 \pm 0.11$ & $2.40 \pm 0.11$ & $1.000 \pm 0.000$ \\
Random & $0.033 \pm 0.008$ & $0.062 \pm 0.011$ & $1.55 \pm 0.23$ & $1.85 \pm 0.28$ & $1.000 \pm 0.000$ \\
\bottomrule
\end{tabular}
\end{table*}

\paragraph{MVTec: transferring the same principle to partially supervised visual mechanisms.}
The inspection task changes both representation and feedback structure: only two actionable defect types are represented by positive calibration data, while two additional mechanisms remain hidden. PraC attains the highest targeted macro-recall AUC (\(0.133\)), hidden-type AUC (\(0.170\)), actionable discovery count (\(11.85\)), and hidden-type coverage (Table~\ref{tab:cas-mvtec-results}). The result shows that the coverage mechanism is not tied to molecular or materials families; it transfers to a visually represented search space in which useful outcome categories are only partially exposed initially.

\begin{table*}[h]
\centering
\caption{MVTec cable 20-seed results (mean $\pm$ SEM).  All final columns are calibration-excluded values after 30 queries.}
\label{tab:cas-mvtec-results}
\small
\begin{tabular}{@{}lrrrrr@{}}
\toprule
Method & Targeted AUC & Hidden-type AUC & Final targeted macro & Actionable hits & Hidden types covered \\
\midrule
PraC & $0.133 \pm 0.005$ & $0.170 \pm 0.010$ & $0.285 \pm 0.010$ & $11.85 \pm 0.39$ & $1.95 \pm 0.05$ \\
IDS & $0.128 \pm 0.004$ & $0.155 \pm 0.010$ & $0.270 \pm 0.008$ & $11.35 \pm 0.33$ & $1.80 \pm 0.09$ \\
Greedy & $0.107 \pm 0.010$ & $0.118 \pm 0.015$ & $0.237 \pm 0.018$ & $10.00 \pm 0.74$ & $1.80 \pm 0.09$ \\
EIG & $0.111 \pm 0.004$ & $0.129 \pm 0.007$ & $0.232 \pm 0.007$ & $9.80 \pm 0.33$ & $1.55 \pm 0.11$ \\
Random & $0.103 \pm 0.008$ & $0.108 \pm 0.012$ & $0.209 \pm 0.012$ & $8.75 \pm 0.50$ & $1.75 \pm 0.10$ \\
\bottomrule
\end{tabular}
\end{table*}

\paragraph{Statistical tests.} Pairwise comparisons in Table~\ref{tab:cas-paired-tests} use two-sided paired Wilcoxon signed-rank tests on per-seed primary AUC; Holm adjustment is applied across the four displayed PraC-versus-baseline comparisons within each application.  The paired $t$-interval is reported as a descriptive 95\% interval for the mean AUC difference.

\begin{table*}[h]
\centering
\caption{Paired primary-AUC comparisons.  $\Delta$ is PraC minus the named
baseline.  ``Wins/ties'' counts paired seeds with positive/zero difference.
Holm adjustment is within application over the four displayed contrasts.}
\label{tab:cas-paired-tests}
\small
\begin{tabular}{@{}llrrrrr@{}}
\toprule
Domain & Baseline & $\Delta$ AUC & 95\% paired $t$-interval & Wins/ties & $p$ & $p_{\mathrm{Holm}}$ \\
\midrule
BACE & Greedy & $0.0378$ & $[0.0299, 0.0457]$ & 19/0 & $3.81\times10^{-6}$ & $1.53\times10^{-5}$ \\
BACE & IDS & $0.0034$ & $[-0.0006, 0.0075]$ & 10/0 & $0.177$ & $0.177$ \\
BACE & EIG & $0.0123$ & $[0.0082, 0.0164]$ & 19/0 & $4.77\times10^{-5}$ & $9.54\times10^{-5}$ \\
BACE & Random & $0.0307$ & $[0.0239, 0.0375]$ & 19/0 & $3.81\times10^{-6}$ & $1.53\times10^{-5}$ \\
3DSC-95K & Greedy & $0.0173$ & $[-0.0426, 0.0772]$ & 10/1 & $0.457$ & $0.495$ \\
3DSC-95K & IDS & $0.0446$ & $[0.0018, 0.0873]$ & 13/0 & $0.165$ & $0.495$ \\
3DSC-95K & EIG & $0.0590$ & $[0.0103, 0.1077]$ & 13/0 & $0.165$ & $0.495$ \\
3DSC-95K & Random & $0.0631$ & $[0.0096, 0.1165]$ & 12/0 & $0.083$ & $0.330$ \\
MVTec cable & Greedy & $0.0251$ & $[0.0000, 0.0501]$ & 13/0 & $0.070$ & $0.139$ \\
MVTec cable & IDS & $0.0048$ & $[-0.0019, 0.0115]$ & 6/8 & $0.433$ & $0.433$ \\
MVTec cable & EIG & $0.0217$ & $[0.0091, 0.0343]$ & 13/4 & $0.003$ & $0.010$ \\
MVTec cable & Random & $0.0294$ & $[0.0149, 0.0438]$ & 16/0 & $0.001$ & $0.004$ \\
\bottomrule
\end{tabular}
\end{table*}

\section{Experimental Details for Composite Bayesian Optimization}
\label{app:composite-bo}

\paragraph{Scope and experimental principle.}
We evaluate expensive vector-valued simulators $f:\mathcal{X}\rightarrow\mathbb{R}^m$ when the decision maker's scalar utility $g:\mathbb{R}^m\rightarrow\mathbb{R}$ is not exposed to the optimizer.  Instead, the optimizer observes outcomes $y=f(x)$ and preference queries of the form $z(y_i,y_j)\in\{i,j\}$.  The three application domains preserve this compositional structure while changing the physical domain, input dimension, and meaning of the outcomes.  

All reported experiments use 20 paired seeds (0--19).  Within a seed, all methods receive the same initialization, simulator, latent preference oracle, and evaluation budget.  Larger preference values are always better. This target-profile preference utility is deliberately hidden from every acquisition method and is used only to answer pairwise comparisons and score discoveries.

\subsection{Domains and datasets}\label{app:cbo-tasks}

\paragraph{Vehicle design: crashworthiness under competing safety and mass outcomes.}
We use the five-dimensional, three-outcome \texttt{VehicleSafety} benchmark distributed with BoTorch \citep{BoTorchTestFunctions} and originating in the real-world multi-objective suite of \citet{Tanabe2020RealWorld}.  The five design variables are reinforced frontal-frame thicknesses; the outcomes encode vehicle mass, full-frontal collision acceleration (an occupant-trauma proxy), and toe-board intrusion (a cabin-damage proxy) \citep{Daulton2020qEHVI}. There is no universally correct scalarization: reducing mass can conflict with crash protection, and acceptable trade-offs depend on manufacturer, regulator, and occupant priorities.  We therefore treat a preferred crashworthiness profile as latent.  The deterministic simulation oracle is
\begin{equation}
 \begin{aligned}
 g_{\mathrm{veh}}(y)&=-\lVert y-y^\star_{\mathrm{veh}}\rVert_2^2,\\
 y^\star_{\mathrm{veh}}&=(1864.72022,11.81993945,0.2903999384).
 \end{aligned}
 \label{eq:cbo-vehicle-utility}
\end{equation}

\paragraph{Bioprocess design: penicillin production.}
We use BoTorch's seven-input, three-outcome \texttt{Penicillin} simulator, adapted from \citet{Liang2021Penicillin}.  Its inputs are culture volume, biomass concentration, temperature, glucose concentration, substrate feed rate, substrate feed concentration, and hydrogen-ion concentration; its outcomes are negative yield, $\mathrm{CO}_2$, and fermentation time \citep{BoTorchTestFunctions}.  Pharmaceutical process design naturally requires preference because increasing yield, shortening batch time, and reducing emissions are operationally valuable but generally non-commensurate; their relative importance can change with production demand, sustainability policy, and plant constraints.  We use
\begin{equation}
 \begin{aligned}
 g_{\mathrm{pen}}(y)&=-\lVert y-y^\star_{\mathrm{pen}}\rVert_2^2,\\
 y^\star_{\mathrm{pen}}&=(25.935,57.612,935.5).
 \end{aligned}
 \label{eq:cbo-penicillin-utility}
\end{equation}

\paragraph{Infrastructure planning: preference-aware energy allocation.}
We optimize a ten-dimensional distributed-energy-resource (DER) allocation on the IEEE 30-bus network implemented in pandapower \citep{Thurner2018Pandapower,PandapowerOPFDocs}.  An OPF design is constrained by load-flow equations, branch limits, bus-voltage limits, and active/reactive power bounds. Pandapower's standard objective is an additive operating-cost term \citep{PandapowerOPFDocs}.  We consider four diverse criteria since cost, equitable service, and contingency tolerance all reflect policy choices as well as engineering feasibility. Table~\ref{tab:cbo-energy-metrics} states the operational definitions used for those criteria.
\begin{table*}[h]
\centering
\caption{Energy-allocation outcomes.  Arrows give the preferred direction.}
\label{tab:cbo-energy-metrics}
\small
\setlength{\tabcolsep}{4pt}
\begin{tabular}{@{}>{\raggedright\arraybackslash}p{2.55cm}>{\raggedright\arraybackslash}p{5.0cm}>{\raggedright\arraybackslash}p{7cm}@{}}
\toprule
Outcome & Operational definition in the experiment & Motivation  \\
\midrule
Voltage fairness $F_V\uparrow$ & Normalized network-wide voltage-uniformity score, constructed monotonically from across-bus voltage dispersion; values nearer one indicate more uniform delivery. & Complements OPF's explicit bus-voltage feasibility constraints by distinguishing feasible solutions with different spatial uniformity. \\
Total cost $C\downarrow$ & DER capital cost plus reactive-power-support cost. & Frozen coefficients are \texttt{DER\_INSTALL\_COST}$=1$ and \texttt{Q\_COST\_FACTOR}$=1$; additive cost is consistent with the pandapower OPF formulation. \\
Priority service $P\uparrow$ & Weighted service score at designated priority buses. & Frozen priority-bus indices are $\{0,5,10,15\}$ and the priority multiplier is $2$.  This criterion represents an explicit policy preference, not an electrical feasibility constraint. \\
Resilience $R\uparrow$ & Fraction of simulator perturbations for which all bus voltages remain inside the safe operating interval. & A scenario-based reliability surrogate: it evaluates the design under load / network perturbations rather than only at the nominal OPF point \citep{Watson2015Resilience}. \\
\bottomrule
\end{tabular}
\end{table*}
The main-task oracle is
\begin{equation}
 y=(F_V,C,P,R),\qquad
 g_{\mathrm{grid}}(y)=[1,-1,2,1]^\top y.
 \label{eq:cbo-energy-utility}
\end{equation}
Thus higher fairness, priority service, and resilience are preferred, whereas lower cost is preferred. 

The overall task configurations are summarized as Table~\ref{tab:cbo-frozen-settings}

\begin{table*}[h]
\centering
\caption{Composite Bayesian optimization task configurations.  $n_0$ and $c_0$ are the numbers of initial designs and pairwise comparisons; $T$ counts recorded post-initialization rounds.}
\label{tab:cbo-frozen-settings}
\small
\setlength{\tabcolsep}{5pt}
\begin{tabular}{@{}p{3cm}rrrrrrc>{\raggedright\arraybackslash}p{5cm}@{}}
\toprule
Application & $d$ & $m$ & $n_0$ & $c_0$ & $q$ & $T$ & Seeds & Acquisition optimizer \\
\midrule
Vehicle safety & 5 & 3 & 5 & 1 & 2 & 80 & 0--19 & 10 restarts; 100 raw samples \\
Penicillin production & 7 & 3 & 5 & 1 & 2 & 100 & 0--19 & 10 restarts; 100 raw samples \\
Energy allocation & 10 & 4 & 10 & 10 & 2 & 95 & 0--19 & 10 restarts; 100 raw samples \\
\bottomrule
\end{tabular}
\end{table*}

\subsection{Acquisition methods}\label{app:cbo-baselines}

Preference feedback is deterministic in these controlled experiments.  For a queried outcome pair $(y_i,y_j)$, the winner is $i$ iff $g(y_i)>g(y_j)$; the numerical value of $g$ is never supplied to the optimizer.  We fit one \texttt{SingleTaskGP} per outcome and combine them in a \texttt{ModelListGP}.  A \texttt{PairwiseGP} with a Laplace marginal likelihood maps outcome vectors to latent preference.  Models are refit after every round.  Each non-random policy optimizes a pair $q=2$ and then observes both simulator outcomes and one winner--loser comparison.

Let $I_f(x;Y)$ denote information gain about the outcome map at a proposed design, $I_g(y;Z)$ information gain about the preference model at its predicted outcome, and $\mathbb{E}[g(y)]$ the posterior pragmatic value.  The common acquisition family is
\begin{equation}
 a(x)=\alpha\,\mathbb{E}[g(y)]
      +\beta\,I_f(x;Y)
      +\gamma\,I_g(y;Z).
 \label{eq:cbo-prac-acquisition}
\end{equation}

\paragraph{Acquisition methods.}

The five compared policies are: (i) \textsc{Random}, uniform designs;
(ii) \textsc{EIG-only}, with $(\alpha,\beta,\gamma)=(0,1,1)$; (iii) \textsc{PraC w/o IG}, the pragmatic term only, $(1,0,0)$; (iv) \textsc{PraC w/o outcome IG}, which retains preference information gain but removes outcome information gain, $(1,0,1)$; and (v) \textsc{PraC full}, $(1,1,1)$. There is no dynamic scheduler in this experiment class.

We further compare four variants of BOPE \citep{Lin2022BOPE}: (i) \textsc{BOPE-I}, qEUBO throughout, with the outcome realization refreshed rather than frozen; (ii) \textsc{BOPE-II}, a two-stage qEUBO$\rightarrow$qNEI schedule, with a frozen outcome snapshot at the stage boundary; (iii) \textsc{BOPE-III}, a two-stage qNEI$\rightarrow$qEUBO schedule, also using a frozen snapshot; and (iv) \textsc{BOPE-IV}, qNEI throughout, with periodic EUBO-$\zeta$ preference elicitation.

\subsection{Evaluation metrics}\label{app:cbo-metrics}

After round $t$, the primary metric is \textbf{incumbent preference}
\begin{equation}
 B_t=\max_{s\leq t,\,k\in\{1,2\}} g(y_{s,k}),
 \label{eq:cbo-incumbent}
\end{equation}
where larger is better.  We report both final $B_T$ and normalized anytime AUC, $\operatorname{AUC}=(T-1)^{-1}\int_1^T B_t\,dt$, computed by the trapezoidal rule on the stored decision grid.  AUC remains in native utility units and is therefore compared only within a task.   

\subsection{Detailed results}\label{app:cbo-detail-results}

\paragraph{Vehicle safety: preference exploitation dominates once a useful direction is available.}
The pragmatic-only, preference-IG, and full PraC variants all reach a substantially better region than Random and EIG-only (Table~\ref{tab:cbo-main-results}). Their similar endpoints indicate that, for this relatively low-dimensional outcome map, exploiting the learned preference is the dominant source of improvement. Full PraC follows the same trajectory while retaining both epistemic channels, illustrating that the hierarchical acquisition naturally reduces to predominantly pragmatic behavior when additional model refinement has limited marginal value.

\paragraph{Penicillin production: joint refinement improves a more complex process-design problem.}
For the seven-dimensional bioprocess problem, full PraC obtains the highest mean final preference and anytime AUC among the evaluated variants (Table~\ref{tab:cbo-main-results}). Unlike vehicle safety, the mapping from process settings to competing production outcomes provides a richer learning problem, and maintaining all three terms yields the strongest mean trajectory. This illustrates the intended compositionality of the hierarchy: preference-guided optimization can remain coupled to continued learning of both the outcome map and the regret model.

\paragraph{Energy allocation: preference exploration and pragmatic action must work together.}
Energy allocation gives the clearest decomposition of the hierarchical acquisition. Pragmatic-only PraC reaches a mean final preference of \(-1.953\), while EIG-only reaches \(-1.826\): neither exploiting the current preference model without refining it nor acquiring information without pragmatic guidance discovers the high-preference region (Table~\ref{tab:cbo-main-results}). Combining pragmatic selection with preference information gain instead raises the final preference to \(5.919\), with full PraC reaching \(5.859\). Removing outcome information gain has little cost in this task, indicating that the dominant unresolved uncertainty lies in the preference layer. This is precisely the role of hierarchical regret: the outer action is evaluated through an inner learning problem over what outcomes should count as desirable.

\begin{table*}[h]
\centering
\caption{Composite-BO results (mean $\pm$ SEM).  Final is $B_T$ and AUC is the normalized incumbent-preference AUC.  }
\label{tab:cbo-main-results}
\small
\setlength{\tabcolsep}{7pt}
\begin{tabular}{@{}llrrr@{}}
\toprule
Domain & Method & $n$ & Final $B_T$ & Anytime AUC \\
\midrule
Vehicle Safety & Random & 20 & $\num{-2.783e+04}\,\pm\,\num{143}$ & $\num{-2.829e+04}\,\pm\,\num{129}$ \\
 & EIG-only & 20 & $\num{-2.766e+04}\,\pm\,\num{97.2}$ & $\num{-2.817e+04}\,\pm\,\num{109}$ \\
 & PraC w/o IG & 20 & \textbf{$\num{-2.565e+04}\,\pm\,\num{0.852}$} & \textbf{$\num{-2.714e+04}\,\pm\,\num{96.7}$} \\
 & PraC w/o outcome IG & 20 & $\num{-2.577e+04}\,\pm\,\num{83.4}$ & $\num{-2.718e+04}\,\pm\,\num{147}$ \\
 & PraC full & 20 & $\num{-2.574e+04}\,\pm\,\num{68.5}$ & $\num{-2.715e+04}\,\pm\,\num{134}$ \\
\addlinespace
Penicillin Production & Random & 20 & $\num{-3.35e+05}\,\pm\,\num{2.93e+03}$ & $\num{-3.409e+05}\,\pm\,\num{1.57e+03}$ \\
 & EIG-only & 20 & $\num{-3.365e+05}\,\pm\,\num{1.93e+03}$ & $\num{-3.416e+05}\,\pm\,\num{628}$ \\
 & PraC w/o IG & 20 & $\num{-3.372e+05}\,\pm\,\num{1.57e+03}$ & $\num{-3.415e+05}\,\pm\,\num{1.02e+03}$ \\
 & PraC w/o outcome IG & 20 & $\num{-3.364e+05}\,\pm\,\num{899}$ & $\num{-3.412e+05}\,\pm\,\num{736}$ \\
 & PraC full & 20 & \textbf{$\num{-3.32e+05}\,\pm\,\num{3.3e+03}$} & \textbf{$\num{-3.392e+05}\,\pm\,\num{2.31e+03}$} \\
\addlinespace
Energy Allocation & Random & 20 & $\num{-1.247}\,\pm\,\num{0.238}$ & $\num{-1.904}\,\pm\,\num{0.0982}$ \\
 & EIG-only & 20 & $\num{-1.826}\,\pm\,\num{0.166}$ & $\num{-2.297}\,\pm\,\num{0.0995}$ \\
 & PraC w/o IG & 20 & $\num{-1.953}\,\pm\,\num{0.289}$ & $\num{-2.573}\,\pm\,\num{0.127}$ \\
 & PraC w/o outcome IG & 20 & \textbf{$\num{5.919}\,\pm\,\num{0.0697}$} & \textbf{$\num{1.972}\,\pm\,\num{0.415}$} \\
 & PraC full & 20 & $\num{5.859}\,\pm\,\num{0.0802}$ & $\num{1.643}\,\pm\,\num{0.219}$ \\
\bottomrule
\end{tabular}
\end{table*}

\paragraph{Comparison with staged preference optimization.} 
The BOPE variants explicitly allocate preference elicitation and optimization through different schedules. BOPE-I achieves the largest anytime AUC, reflecting faster early progress, while PraC achieves the highest mean final incumbent and final selected design (Table~\ref{tab:cbo-bope-results}). PraC reaches this performance using the same acquisition rule throughout the run rather than introducing an external stage transition. This comparison illustrates the practical consequence of the hierarchical construction: preference learning and optimization can be composed continuously as parts of one decision rule.

\begin{table*}[h]
\centering
\caption{BOPE comparison (20 seeds; mean $\pm$ SEM).}
\label{tab:cbo-bope-results}
\small
\setlength{\tabcolsep}{10pt}
\begin{tabular}{@{}lrrr@{}}
\toprule
Method & Final best & Best AUC & Final selected \\
\midrule
BOPE-I & $\num{6.9}\,\pm\,\num{0.1}$ & $\num{6.634}\,\pm\,\num{0.142}$ & $\num{5.1}\,\pm\,\num{0.748}$ \\
BOPE-II & $\num{6.3}\,\pm\,\num{0.2}$ & $\num{6.037}\,\pm\,\num{0.213}$ & $\num{4.6}\,\pm\,\num{0.292}$ \\
BOPE-III & $\num{6}\,\pm\,\num{0.447}$ & $\num{5.543}\,\pm\,\num{0.465}$ & $\num{4.3}\,\pm\,\num{0.122}$ \\
BOPE-IV & $\num{5.8}\,\pm\,\num{0.339}$ & $\num{5.595}\,\pm\,\num{0.355}$ & $\num{4.8}\,\pm\,\num{0.406}$ \\
PraC & $\num{7}\,\pm\,\num{0}$ & $\num{6.248}\,\pm\,\num{0.216}$ & $\num{5.4}\,\pm\,\num{0.992}$ \\
Random & $\num{5.1}\,\pm\,\num{0.187}$ & $\num{4.6}\,\pm\,\num{0.0965}$ & $\num{2.8}\,\pm\,\num{0.374}$ \\
\bottomrule
\end{tabular}
\end{table*}

\paragraph{Paired comparisons.} We make paired, two-sided Wilcoxon signed-rank comparisons on per-seed AUC and Holm-adjust the four PraC-full-versus-ablation tests within each task in Table~\ref{tab:cbo-paired-tests}.

\begin{table*}[h]
\centering
\caption{Paired anytime comparisons.  $\Delta$ is PraC-full AUC minus the listed method's AUC; brackets give a descriptive paired 95\% $t$ interval. $p_{\mathrm{Holm}}$ is the two-sided Wilcoxon value after four comparisons within each application.}
\label{tab:cbo-paired-tests}
\small
\begin{tabular}{@{}llrrr@{}}
\toprule
Application & Comparator & $n$ & $\Delta$ AUC [95\% CI] & $p_{\mathrm{Holm}}$ \\
\midrule
Vehicle Safety & Random & 20 & $\num{1140}$ $[\num{838.8},\,\num{1441}]$ & $\num{2.29e-05}$ \\
 & EIG-only & 20 & $\num{1015}$ $[\num{725.7},\,\num{1305}]$ & $\num{0.000143}$ \\
 & PraC w/o IG & 20 & $\num{-12.93}$ $[\num{-349.6},\,\num{323.7}]$ & $1.000$ \\
 & PraC w/o outcome IG & 20 & $\num{21.53}$ $[\num{-376.8},\,\num{419.9}]$ & $1.000$ \\
\addlinespace
Penicillin Production & Random & 20 & $\num{1703}$ $[\num{-3953},\,\num{7359}]$ & $1.000$ \\
 & EIG-only & 20 & $\num{2446}$ $[\num{-2825},\,\num{7717}]$ & $1.000$ \\
 & PraC w/o IG & 20 & $\num{2292}$ $[\num{-3426},\,\num{8009}]$ & $1.000$ \\
 & PraC w/o outcome IG & 20 & $\num{2014}$ $[\num{-2686},\,\num{6714}]$ & $1.000$ \\
\addlinespace
Energy Allocation & Random & 20 & $\num{3.547}$ $[\num{2.982},\,\num{4.112}]$ & $0.250$ \\
 & EIG-only & 20 & $\num{3.94}$ $[\num{3.232},\,\num{4.648}]$ & $0.250$ \\
 & PraC w/o IG & 20 & $\num{4.216}$ $[\num{3.398},\,\num{5.033}]$ & $0.250$ \\
 & PraC w/o outcome IG & 20 & $\num{-0.3292}$ $[\num{-1.721},\,\num{1.063}]$ & $0.625$ \\
\bottomrule
\end{tabular}
\end{table*}

\end{document}

%% file: math_commands.tex
\usepackage{amsmath,amsfonts,bm}

\def\eqref#1{equation~\ref{#1}}
\def\1{\bm{1}}

\DeclareMathAlphabet{\mathsfit}{\encodingdefault}{\sfdefault}{m}{sl}
\SetMathAlphabet{\mathsfit}{bold}{\encodingdefault}{\sfdefault}{bx}{n}